\documentclass[11pt]{article}

\usepackage{amsmath}
\usepackage{amsthm}
\usepackage{amsfonts}
\usepackage{url}
\usepackage{fullpage}

\usepackage{authblk} 
\usepackage{lipsum}
\usepackage[normalem]{ulem}
\newcommand{\RR}{\mathbb{R}}

\newcounter{commentcounter}
\long\def\symbolfootnote[#1]#2{\begingroup%
\def\thefootnote{\fnsymbol{footnote}}\footnote[#1]{#2}\endgroup}

\newif\iftextinsteadofmath
\newcommand{\defv}[2]{
\iftextinsteadofmath
    \expandafter\newcommand\csname#1\endcsname{{\sf #1}}
\else
    \expandafter\newcommand\csname#1\endcsname{#2}
\fi
}

\newcommand{\redefv}[2]{
\iftextinsteadofmath
    \expandafter\renewcommand\csname#1\endcsname{{\sf #1}}
\else
    \expandafter\renewcommand\csname#1\endcsname{#2}
\fi
}

\newcommand{\btab}{\begin{tabbing}}
\newcommand{\etab}{\end{tabbing}}
\newcommand{\ben}{\begin{enumerate}}
\newcommand{\een}{\end{enumerate}}
\newcommand{\be}{\begin{equation}}
\newcommand{\ee}{\end{equation}}
\newcommand{\bea}{\begin{eqnarray}}
\newcommand{\eea}{\end{eqnarray}}
\newcommand{\bean}{\begin{eqnarray*}}
\newcommand{\eean}{\end{eqnarray*}}
\newcommand{\beal}[1]{\begin{alignat}{#1}}
\newcommand{\eeal}{\end{alignat}}
\newcommand{\bealn}[1]{\begin{alignat*}{#1}}
\newcommand{\eealn}{\end{alignat}}
\newcommand{\bd}{\begin{description}}
\newcommand{\ed}{\end{description}}
\newcommand{\bi}{\begin{itemize}}
\newcommand{\ei}{\end{itemize}}
\newcommand{\bc}{\begin{center}}
\newcommand{\ec}{\end{center}}
\newcommand{\bq}{\begin{quotation}}
\newcommand{\eq}{\end{quotation}}
\newcommand{\bco}{\begin{comment}}
\newcommand{\eco}{\end{comment}}
\newcommand{\bmat}{\left(\begin{matrix}}
\newcommand{\emat}{\end{matrix}\right)}

\newcommand{\ba}[1]{\begin{array}{#1}}
\newcommand{\ea}{\end{array}}

\newcommand{\bfig}{\begin{figure}[htb] \rule{\linewidth}{0.1mm}\\ }
\newcommand{\efig}{ \rule{\linewidth}{0.1mm} \end{figure}}

\newcommand{\bmp}[1]{\begin{minipage}{#1}}
\newcommand{\bmpp}[2]{\begin{minipage}[#1]{#2}}
\newcommand{\emp}{\end{minipage}}

\newcommand{\norm}[1]{{\lVert#1 \rVert}}

\usepackage{xcolor} 
\usepackage{graphicx}
\graphicspath{ {images/} }
\usepackage{subfigure}
\usepackage{enumitem}
\usepackage[english]{babel}
\usepackage[utf8]{inputenc}
\usepackage[T1]{fontenc}
\usepackage{multirow}
\usepackage{epstopdf}
\usepackage{hyperref}
\hypersetup{colorlinks, citecolor = [rgb]{0, 0, 1}, urlcolor = [rgb]{0.5, 0, 0}, linkcolor = [rgb]{0, 0.5, 0}}

\allowdisplaybreaks

\theoremstyle{plain}
\newtheorem{thm}{Theorem}[section]

\newtheorem{asmp}[thm]{Assumption}
\newtheorem{lem}[thm]{Lemma}
\newtheorem{prop}[thm]{Proposition}
\newtheorem{defn}{Definition}[section]
\newtheorem{rem}{Remark}[section]

\title{Differentially Private Accelerated Optimization Algorithms \thanks{Nurdan Kuru is the first author; the other authors contributed equally. \newline Email: \url{nurdankuru@sabanciuniv.edu}, \url{birbil@ese.eur.nl}, \url{mg1366@rutgers.edu}, \url{sinanyildirim@sabanciuniv.edu}}}

\author[$\dagger$]{Nurdan Kuru}
\author[$\star$]{Ş. İlker Birbil}
\author[$\ddag$]{Mert Gürbüzbalaban}
\author[$\dagger$]{Sinan Yıldırım}

\affil[$\dagger$]{Faculty of Engineering and Natural Sciences, Sabancı University}
\affil[$\star$]{Econometric Institute, Erasmus University, Rotterdam}
\affil[$\ddag$]{Department of Management Science and Information Systems, Rutgers University.}

\date{\today}

\begin{document}
\renewcommand\footnotemark{}
\maketitle

% REQUIRED
\begin{abstract}
 We present two classes of differentially private optimization algorithms derived from the well-known accelerated first-order methods. The first algorithm is inspired by Polyak's heavy ball method and employs a smoothing approach to decrease the accumulated noise on the gradient steps required for differential privacy. The second class of algorithms are based on Nesterov's accelerated gradient method and its recent multi-stage variant. We propose a noise dividing mechanism for the iterations of Nesterov's method in order to improve the error behavior of the algorithm. The convergence rate analyses are provided for both the heavy ball and the Nesterov's accelerated gradient method with the help of the dynamical system analysis techniques. Finally, we conclude with our numerical experiments showing that the presented algorithms have advantages over the well-known differentially private algorithms.
\end{abstract}

\section{Introduction} \label{sec: Introduction}

In many real applications involving data analysis, the data owners and the data analyst may be different parties. In such cases, privacy of the data could be a major concern. Differential privacy promises securing an individual's data while still revealing useful information about a population \cite{Dwork_2006}. It is based on constructing a mechanism, for which output stays probabilistically similar whenever a new item is added or an existing one is removed from the data set. Such incremental mechanisms have been shown to ensure data privacy \cite{Dwork_2008}. 
% Clearly, privacy is not the utmost interest when the data owner and the data analyst are the same.
Differential privacy is used within various types of methods in machine learning, such as; boosting, linear and logistic regression and support vector machines \cite{Dwork_etal_2010, Chaudhuri_Monteleoni_2009, rubinstein2009learning, zhang2012functional}.

In this work, we consider the scenario where a data analyst performs analysis on a dataset owned by another party by means of solving an optimization problem with (stochastic) first-order methods for empirical risk minimization. There is in fact a large body of work on differentially private empirical risk minimization \cite{chaudhuri2011differentially, kifer2012private, Bassily_etal_2014, Zhang_etal_2017}. We will specifically focus on privacy preserving gradient-based iterative algorithms, which are a popular choice for large-scale problems due to their scalability properties \cite{abadi2016,  shokri2015privacy, yu2019differentially}. Our contributions specifically regard two gradient-based stochastic accelerated algorithms, Polyak's heavy ball (HB) algorithm \cite{polyak1964some}, and Nesterov's accelerated gradient (NAG) algorithm \cite{nesterov1983} as well as a recent variant of NAG \cite{aybat2019masg}.

Differential privacy can be achieved by adding carefully adjusted random noise to the input (data) such as in \cite{Foulds_et_al_2016}, to the output (some function of data), such as in \cite{Chaudhuri_Monteleoni_2009}, or to the iteration vectors of an iterative algorithm, such as in \cite{abadi2016, pichapati2019adaclip}. In this paper, we focus on the latter case in connection with gradient-based algorithms, where the iteration vectors of a gradient-based algorithm are revealed at the intermediate steps. This scenario is particularly relevant, for example, when some assessment should done publicly on the convergence of the algorithms, or when the available data are shared among multiple users. Although the intrinsic randomness in a stochastic gradient descent algorithm has been shown to provide some level of privacy in a recent study \cite{hyland2019intrinsic}, the authors report high levels of privacy loss for most datasets.
%  (e.g. $\epsilon = 13.01$). 
That is why most of the studies in the literature consider adding a suitable noise vector to the gradient at each step. However, this noise does harm the performance of the algorithm in such a way that it may even cause divergence. Therefore, the utility of a privacy preserving algorithm is always a concern, as in our work. 

There is a large amount of work for improving the utility of gradient based algorithms while preserving a given amount of privacy `budget' (a mathematical definition of this budget is given in Section \ref{sec: Preliminaries}).  A well known computational tool is, for example, subsampling, which is analyzed in a broader context in \cite{Bassily_etal_2014}. Norm clipping, that is, bounding the norm of the gradient according to a threshold, is also used to control the amount of noise; see for instance \cite{abadi2016, pichapati2019adaclip, song2020characterizing}. Analytical developments are also present: The authoros of \cite{abadi2016} focus on tracking higher moments of the privacy loss to obtain tighter estimates on the privacy loss. Other forms of differential privacy are also employed to conduct tighter analysis of the privacy loss \cite{Dwork_and_Rothblum_2016, Bun_and_Steinke_2016, Mironov_2017, Zhang_etal_2017}. 

\smallskip
% \paragraph{\textbf{Contributions}} \label{sec: Contribution}
\textbf{Contributions:} In this paper, we contribute to the existing literature on privacy preserving gradient-based algorithms by proposing, and providing a theoretical analysis of, differentially private versions of HB and NAG. 
% Those algorithms are well known for improved convergence properties over the basic gradient descent algorithm. We propose to use their stochastic versions in the differential privacy setting, and demonstrate their advantages with theoretical and numerical findings. 
%In the following, we will explain our contributions for HB and NAG separately.

% A related question is how to distribute the privacy noise 
Our first algorithm is a variant of HB, which employs a smoothing approach by the help of the information from the previous iterations. We use this mechanism to improve the privacy level by taking the weighted average of the current and the previous noisy gradients. We give a convergence rate analysis using the dynamical system analysis techniques for optimization algorithms \cite{lessard2016analysis,  hu2017dissipativity, fazlyab2018analysis}. Although this kind of analysis exists for the deterministic HB method \cite{hu2017dissipativity}, to the best of our knowledge, the case with noisy gradients has not been considered in the literature, except in \cite{can2019accelerated}, where a special case of quadratic objectives is studied for a particular choice of the stepsize and the momentum parameter (corresponding to the traditional choice of parameters in deterministic HB methods). By extending on \cite[Theorem 12]{can2019accelerated}, we give general results in terms of the error bounds for any selection of stepsize and momentum parameters.

The main motivation behind our error analysis is to shed light on the effect of the free parameters in the algorithm, such as the stepsize and the momentum parameters, and the number of iterations, on the performance. In the typical stochastic optimization setting, the noise in the gradients is assumed to have a bounded variance which does not depend on the number of iterations, therefore the performance bounds obtained for the accuracy of momentum-based algorithms such as NAG or HB (measured in terms of expected suboptimality of the iterates) with constant parameters can improve monotonically as the number of iterations is increased (see e.g. \cite{lan2020first,can2019accelerated,aybat2019masg}). However, this is not necessarily the case in privacy preserving versions of these algorithms.
% \sout{Unlike in a normal setting, the accuracy performances of iterative privacy preserving algorithms do not monotonically improve with an increase in the number of iterations}.
% What happens typically is an improvement in the accuracy up to a certain number of iterations, which is followed by a degrading in the accuracy as the number of iterations increases. 
This is because each iteration causes some privacy loss and the amount of noise in the gradients has to be increased as the total number of iterations increases. Likewise, it is not clear how to set the stepsize and the momentum parameter for optimum performance of a privacy preserving version of an algorithm because of the complex trade-off between the convergence rate and additive error due to noise. We address such issues for the differentially private HB algorithm by providing performance bounds and error rates
%\comment{Error rate derken suboptimality mi demek istiyoruz?} 
in terms of the number of iterations and the momentum parameters. We extend the existing results from the literature \cite{hu2017dissipativity, can2019accelerated} to provide an analysis for general stepsize and momentum parameter choices for both quadratic objectives as well as for smooth strongly convex objectives for the HB method under noisy gradients. In particular, tuning the stepsize and the momentum parameters to the level of desired privacy level allows us to achieve better accuracy in the privacy setting compared to traditional choice of parameters previously used for the deterministic HB method.

Our second contribution regards differentially private versions of NAG \cite{nesterov1983}. NAG can simply be made differentially private by merely adding noise to the its gradient calculations. However, how to distribute the privacy preserving noise to iterations to have an optimal performance has not been concretely addressed in the literature. This question can be reformulated as \emph{how to distribute a given, fixed, privacy budget to the iterations of the algorithm}. The relevance of this question is due to the fact that in each iteration a noisy gradient is revealed, causing privacy loss. We address this problem for the differentially private versions of NAG. In doing so, we exploit some explicit bounds in \cite{aybat2019masg} on the expected error of those algorithms when they are used with noisy gradients. Our findings show that distributing privacy budget to iterations uniformly, which corresponds to using the same variance for the privacy preserving noise for all iterations, is not the optimal way in terms of accuracy.

We also consider a differentially private version of a recent variant of NAG, the multi-stage accelerated stochastic gradient (MASG), introduced in \cite{aybat2019masg} to improve error behavior. The method is tailored to deal with noisy gradients in NAG, hence is quite relevant to our setting in which noise is used to help with preserving privacy. However, the authors have not considered differential privacy while designing their algorithm. Similar techniques to NAG will be used for the error analysis of the differentially private version of MASG. Moreover, our novel scheme of optimally distributing the privacy budget to the iterations can also be applied to MASG in a similar manner.

We would like to mention the techniques for, and the scope of, the analysis of our proposed algorithms. By their nature, the proposed algorithms are stochastic, where the gradient vector is augmented with privacy preserving noise at each iteration. There exist several studies that analyze the convergence of stochastic accelerated algorithms; for instance, see \cite{loizou2017momentum, gadat2018stochastic, ramezani2018stability} for works related to stochastic HB, and \cite{Yang_etal_2016, yan2018unified, Mohammadi_et_al_2020} for a unified analysis of stochastic versions of GD, NAG and HB methods. We adopt a dynamical system representation approach that is preferred to analyze the first order optimization algorithms \cite{lessard2016analysis, fazlyab2018analysis, hu2017dissipativity, aybat2018robust, aybat2019masg,Jovanovic, Mohammadi_et_al_2020}. In this approach, the convergence rate is found with respect to the rate of decrease of a Lyapunov function of the system state of the dynamic system induced by the algorithm.

Finally, we remark that the given results are satisfied even when the noise that corrupts the gradient is \textit{uncorrelated} with the state of the algorithm, provided that the noise variance can be bounded. The case of uncorrelated noise is evidently more general than the case of independent noise. In our setting, uncorrelatedness of the noise in the gradient is ensured by the noise being  zero mean with a bounded variance conditioned on the state of the algorithm. Such characteristics of the gradient noise is quite relevant to differential privacy for two reasons: First, subsampling is a common technique used in privacy preserving algorithms, and the error due to subsampling has zero mean and its variance is typically dependent on the current iterate of the algorithm. Second, the variance of the privacy preserving noise is adjusted by a so-called sensitivity function of the state of the algorithm, which may be state dependent.

\section{Preliminaries} \label{sec: Preliminaries}

A vast variety of problems in machine learning can be written as unconstrained optimization problems of the form 
\begin{equation} \label{eq: main_prob}
\min_{x \in \mathbb{R}^{d}} F(x),
\end{equation}
where $x \in \mathbb{R}^{d}$ is a parameter vector of dimension $d \geq 1$. This papers concerns a data-oriented optimization problem, where the objective function depends on a given dataset $Y = \{ y_{1}, \ldots, y_{n} \} \subseteq \mathcal{Y}$. The objective function in \eqref{eq: main_prob} is a sum of functions that correspond to contributions of the individual data points $y_{1}, \ldots, y_{n}$ to the global objective. More specifically, we are interested in objective functions of the form
\begin{equation}
F(x) = \frac{1}{n} \sum_{i  = 1}^{n} f(x; y_{i}),
\label{eq-finite-sum}
\end{equation}
where $f(\cdot; y): \mathbb{R}^{d} \mapsto \mathbb{R}$ for $y \in \mathcal{Y}$. These problems arise in empirical risk minimization in the context of supervised learning \cite{vapnik2013nature}. Note that one could write $F(x; Y)$ in order to emphasize the dependency of $F$ on $Y$. However, for the sake of simplicity, we suppress $Y$ in the notation. In this paper, we further restrict our attention to the set of strongly convex and smooth (that is with a Lipschitz continuous gradient) functions; see Definition \ref{defn: st_convex} in Appendix \ref{app:0}.

Gradient-based methods are arguably the most popular methods for the optimization problem in \eqref{eq: main_prob}. We define the gradient vectors for the additive functions
\[
\nabla f(x; y) = \left( \frac{\partial f(x; y)}{\partial x_{1}}, \ldots, \frac{\partial f(x; y)}{\partial x_{d}} \right)^{\top}, \quad x \in \mathbb{R}, \quad y \in \mathcal{Y},
\]
so that the \emph{(full)} gradient $\nabla F(x)$ is given by
\begin{equation} \label{eq: full gradient}
\nabla F(x) = \frac{1}{n}\sum_{i=1}^{n} \nabla f(x; y_{i}).
\end{equation}
The iterates of the basic gradient descent method for the solution of \eqref{eq: main_prob} is given by
\begin{equation} \label{eq: GD basic}
x_{t+1} = x_{t} - \alpha \nabla F(x_{t}), \quad t \geq 1,
\end{equation}
where $\alpha$ is the (constant) learning rate. There are two well-known modifications of the basic gradient descent; Polyak's heavy ball (HB) method \cite{polyak1964some} and Nesterov's accelerated gradient (NAG) method \cite{nesterov1983}. Both introduce a momentum parameter $\beta\geq 0$ to improve upon the convergence of gradient descent. The update rule for HB at iteration $t$ is given by
\begin{equation}
\label{eq: HB Basic}
x_{t+1} = x_{t} - \alpha \nabla F(x_{t}) + \beta (x_{t}-x_{t-1}),
\end{equation}
whereas the update rule for NAG at iteration $t$ is simply
\begin{align}\label{eq: NAG Basic}
\begin{split}
& x_{t+1}  = z_{t} - \alpha \nabla F(z_{t}), \\
& z_{t} = (1+\beta) x_{t} - \beta x_{t-1}.
\end{split}
\end{align}
There exist stochastic versions of these gradient-based methods that are employed when either the gradients are noisy or an exact calculation per iteration is too expensive. In the former case, $\nabla F(x_{t})$ is simply replaced by the noisy gradient, provided that the noisy gradient is an unbiased estimator of the true gradient. In the second case, the computationally costly $\nabla F(x_{t})$ is replaced by a mini-batch estimator 
\begin{equation} \label{eq: gradient of the mini-batch}
\nabla F_{B_t}(x) := \frac{1}{m}\sum_{i \in B_t} \nabla f(x; y_{i}),
\end{equation}
where $B_t$ is a subset $B \subseteq \{1, \dots, n\}$ with $|B_t| = m$, formed by sampling without replacement so that $\nabla F_{B_t}(x)$ is unbiased.

In our subsequent discussion, we will modify the steps of the gradient-based methods to have privacy-preserving updates. Our setting is as follows: The data holder makes public the iterates
$\{ x_{t} \}_{0 \leq t \leq T}$ for a total of $T$ iterations. The algorithm is known with all its parameters $\alpha$ (and $\beta$). If the data holder applies the related update of the method directly, the vectors $ \nabla F(x_{t})$ are revealed. This violates privacy since the revealed terms are deterministic functions of the data. Therefore, due to privacy concerns, the iterates have to be randomized by using a noisy gradient.

Differential privacy quantifies the privacy level that one guarantees by such randomizations. A randomized algorithm takes an input dataset $Y \in \mathcal{Y}$ and returns the random output $A_{Y} \in \mathcal{X}$. Such an algorithm can be associated with a function $\mathcal{A}: \mathcal{Y} \rightarrow \mathcal{P}$ that maps a dataset from $\mathcal{Y}$ to a probability distribution $\mathcal{A}(Y) \in \mathcal{P}$ such that the output is random with $A_{Y} \sim \mathcal{A}(Y)$. For datasets $Y_{1}$ and $Y_{2}$, let $h(Y_{1}, Y_{2})$ denote the Hamming distance between $Y_{1}$ and $Y_{2}$. This distance indicates the number of different elements between the two datasets. A differentially private algorithm ensures that $\mathcal{A}(Y_{1})$ and $\mathcal{A}(Y_{2})$ are ``not much different,''  if $h(Y_{1}, Y_{2}) = 1$. This statement is formally expressed by \cite{Dwork_2008} (see Definition \ref{defn:DP} in Appendix \ref{app:0}).

Most existing differentially private methods perturb certain functions of data with a suitably chosen random noise. The amount of this noise is related to the sensitivity of the function, which is the maximum amount of change in the function when one single entity of the data is changed (see Definition \ref{defn:Sensitivity} in Appendix \ref{app:0}). There are many results proposed in the literature that provide differential privacy for iterative algorithms. Among those results, we will mainly use three of them concerning Laplace mechanism, composition and subsampling. For ease of reference, corresponding three theorems are also given in Appendix \ref{app:0}.

When privacy is of concern for the optimization problem \eqref{eq: main_prob}, one approach is to update the parameter $x_{t}$ of iteration $t$ using a noisy (stochastic) gradient vector 
\begin{equation} \label{eq: noisy gradient vector}
\widetilde{\nabla F_{B_{t}}}(x_{t}) = \nabla F_{B_{t}}(x_{t}) + \eta_{t},
\end{equation}
where $B_{t}$ is the indices of full (sampled) data with size $m$ and $\eta_{t} = \left(\eta_{t, 1}, \ldots, \eta_{t, d} \right)^{\top}$ is a vector of independent noise terms having Laplace distribution with its parameter value chosen suitably to provide the desired level of privacy. Although the privacy of an algorithm can be guaranteed in this way, the performance will be affected because of the noise added at each iteration. In this paper, we analyze the present trade-offs between accuracy and privacy in gradient based algorithms, and propose accelerated algorithms with good performance under the differential privacy noise.

\section{Differentially Private Heavy Ball Algorithm} \label{sec: Differentially Private Heavy Ball Algorithm}

We start with investigating a differentially private version of the stochastic HB algorithm, which we will abbreviate as DP-SHB. The update rule of this algorithm operates on a dataset of size $n$ with steps 
\begin{equation} \label{eq: DP-SHB update rule}
x_{t+1} = x_{t} - \alpha (\nabla F_{B_{t}}(x_{t}) + \eta_{t}) + \beta (x_{t}-x_{t-1}),
\end{equation}
where $0 < \beta < 1$ is the momentum parameter of HB, $B_{t}$'s are i.i.d.\ random subsamples of size $m \leq n$ sampled without replacement, and $\eta_{t}$'s are independent random vectors having i.i.d.\ noise components with $\textup{Laplace}(b_{t}(x_{t}))$ which is the Laplace distribution with a zero mean and variance $2b_{t}(x_{t})^2$. Here, differential privacy of \eqref{eq: DP-SHB update rule} is sought through the noisy gradient $\nabla F_{B_{t}}(x_{t}) + \eta_{t}$. The minimum value, required for the parameter $b_{t}(x_{t})$ of the Laplace distribution to have $\epsilon$-differential privacy, depends on the number of iterations $T$, the subsample size $m$, and the $L_{1}$ sensitivity $S_{1}(x_{t})$ at $x_{t}$, where the $L_{1}$ sensitivity function is defined as
\begin{equation}\label{eq: L1 sensitivity of f}
S_{1}(x) = \sup_{y, y' \in  \mathcal{Y}} | \nabla f(x; y) - \nabla f(x; y') |, \quad  x \in \RR^{d}.
\end{equation}
Observing \eqref{eq: gradient of the mini-batch}, we see that changing $Y$ and $Y'$ in one data item corresponds to the existence of a single pair of different values $(y_{i}, y'_{i})$. Hence, the change in $F(x)$ is by at most $S_{1}(x)/n$.

Consider the DP-SHB algorithm, where at iteration $t$, we draw a subsample of size $m$ from a dataset of size $n$, and add Laplacian noise with parameter $b_{t}(x_{t})$ to the mini-batch estimator in \eqref{eq: noisy gradient vector}. Then, using the result regarding the Laplace mechanism in Theorem \ref{thm: Laplace mechanism}, and the privacy amplification result stated in Theorem \ref{thm: subsampling_last}, the privacy leak at the iteration can be shown to be
\[
\epsilon_{t} = \varepsilon(S_{1}(x_{t}), b_{t}(x_{t}), n, m),
\]
where the function $\varepsilon: [0, \infty)^{2} \times \{ (m, n) \in \mathbb{Z}_{+}: m \leq n \} \mapsto \RR$ is given as
\begin{equation} \label{eq: epsilon under subsampling}
\varepsilon(S, b, n, m) := \ln \left[ (e^{S/(bm)} - 1)\frac{m}{n} + 1 \right], \quad  \text{for} \quad S, b \in [0, \infty)^{2}; m \leq n \in \mathbb{Z}_{+}.
\end{equation}
Note that, for $m = n$, i.e., under no subsampling, we end up with $\varepsilon(S, b, n, n) = S/(bn)$. The following proposition uses this fact and states the required amount of noise variance in order to have an $\epsilon$-differentially private algorithm after $T$ iterations.
\begin{prop} \label{prop: algstocDP} The DP-SHB algorithm in \eqref{eq: DP-SHB update rule} leads to an $\epsilon$ differentially private algorithm if the parameter $b_{t}(x_{t})$ of the Laplace distribution $\textup{Laplace}(b_{t}(x_{t}))$ for each component of the noise vector $\eta_{t}$ at iteration $t$ is chosen as
\begin{align} \label{eq: required stoc sigma} b_{t}(x_{t}) = \frac{S_{1}(x_{t})}{m \epsilon_{0}},
\end{align}
where $x_{t}$ is the output value at iteration $t$, $n$ is the number of data points,
\begin{equation} \label{eq: epsilon per iteration}
\epsilon_{0} = \ln \left[ 1 + ( e^{\epsilon/T} - 1 ) n/m \right],
\end{equation}
$m$ is the subsample size, and $T$ is the maximum number of iterations.
\end{prop}
\begin{proof}
Using the $b_{t}(x_{t})$ given in the proposition, the privacy loss in one iteration is, $\varepsilon(S_{1}(x), b_{t}(x_{t}), n, m) = \ln \left[1 + (m/n) \left( e^{\epsilon_{0}} - 1 \right) \right] = \epsilon/T$. Finally, we apply Theorem \ref{thm: composition} to conclude that the privacy loss after $T$ iterations is $\epsilon$.
\end{proof}

We are interested in DP-SHB, for it lends itself to an interpretation quite relevant to the differential privacy setting. The noise used in the differentially private versions of the gradient descent algorithm has to be higher as the number of iterations grows, i.e., $b_t(x_t)$ needs to be larger for a larger $T$. This can be seen from equation \eqref{eq: required stoc sigma}. One way to reduce the required noise is to use a smoothed noisy gradient, where the smoothing is recursively performed on the past and the current gradient estimates. This is indeed how DP-SHB works. The update in \eqref{eq: DP-SHB update rule} can be rewritten as
\begin{align} \label{eq: update with smoothed noisy gradient}
x_{t+1} = x_{t} - \frac{\alpha}{1 - \beta} \bar{u}_t,  %{u_{t}},
\end{align}
where $\bar{u}_t$ is a geometrically weighted average of all the gradients up to the current iteration defined recursively as
\begin{align} \label{eq: smoothed gradient}
\bar{u}_{t} = \beta \bar{u}_{t-1} + (1 - \beta) \left( \nabla F_{B_{t}}(x_{t}) + \eta_{t} \right)
\end{align}
with the initial condition $\bar{u}_{0} = (1-\beta) (\nabla F_{B_{0}}(x_{0}) + \eta_{0})$. We note that a similar smoothing strategy as in DP-SHB, which combines mini-batching with a noise-adding mechanism for averaged gradients, has been used in \cite{Park_etal_2016}; however in a different setting, namely for the purpose of private variational Bayesian inference.

\subsection{Analysis of DP-SHB}
\label{sec: Error analysis of DP-SHB}
For analyzing the convergence of DP-SHB, we first cast it as a dynamical system. We introduce the (random) variable
\[
v_{t} =  \nabla F(x_{t}) - \nabla F_{B_{t}}(x_{t}),
\]
which accounts for the error due to subsampling. Using this definition, we can write
\begin{align}
x_{t+1} = x_{t} - \alpha (\nabla F(x_{t}) + \eta_{t} + v_{t}) +  \beta (x_{t} - x_{t-1}).
\end{align}
Then, the dynamical system representation of DP-SHB becomes
\begin{equation}
\label{eq: DP-SHB dynamic system representation}
\begin{split}
\xi_{t+1} &=[A \otimes I_{d}] \xi_{t} +  [B \otimes I_{d}] (u_{t}+ v_{t} + \eta_{t}),  \\
z_{t} &=[C \otimes I_{d}] \xi_{t}, \\
u_{t} &=\nabla F(z_{t}),
\end{split}
\end{equation}
where $I_d$ is the $d\times d$ identity matrix, $\otimes$ denotes the Kronecker product, and the state vector $\xi_{t}$ and the system matrices $A$, $B$, and $C$ are given as
\begin{equation} \label{eq: DP-SHB system matrices} 
\xi_{t} = \begin{bmatrix} x_{t} \\ x_{t-1} \end{bmatrix}, \quad
A = \begin{bmatrix} 1+\beta & -\beta  \\ 1 & 0 \end{bmatrix}, \quad B = \begin{bmatrix} \alpha  \\ 0 \end{bmatrix}, \quad C = \begin{bmatrix} 1 & 0 \end{bmatrix}.
\end{equation}
 
% Note that we have $x_t = [I_d \quad 0_d] \xi_t$. 
In our error analysis, we will consider both stochastic and deterministic versions of HB. In order to do that, we need a uniform bound (in $x_{t}$) for the conditional covariance of $w_{t} := \eta_{t}+v_{t}$ given $x_{t}$ (for the case without subsampling, we simply take $v_{t} = 0$). Note that, due to independence of $\eta_{t}$ and $v_{t}$ conditional on $x_{t}$, the conditional covariance of $w_t$ given $x_t$ satisfies 
\[
\textup{Cov}(w_{t} | x_{t}) = \textup{Cov}(\eta_{t} | x_{t})+ \textup{Cov}(v_{t} | x_{t}).
\]
To handle the contribution to the overall noise by the privacy preserving noise $\eta_{t}$, we make the following assumption.
\begin{asmp}[Bounded $L_{1}$ sensitivity] \label{asmp: Bounded L1 sensitivity} The $L_{1}$ sensitivity function defined in \eqref{eq: L1 sensitivity of f} is bounded in $x$. That is, there exists a scalar constant $S_{1}$ such that
\begin{equation} \label{eq: assumption on L1}
\sup_{x \in \RR^{d}} S_{1}(x) \leq S_{1}.
\end{equation}
\end{asmp}
Assumption \ref{asmp: Bounded L1 sensitivity} is common in the differential privacy literature. For example, the logistic regression model, which we will use to show our numerical experiments in Section \ref{sec: Experimental results}, easily admits such a bound. It turns out that Assumption \ref{asmp: Bounded L1 sensitivity} readily guarantees a bound on the variance of $v_{t}$, the subsampling noise.  The next proposition formally shows this observation. The proof is given in Appendix \ref{app:A}.

\begin{prop} \label{prop: bound on the variance of subsampling} If Assumption \ref{asmp: Bounded L1 sensitivity} holds, the norm of the conditional covariance of $w_{t} = \eta_{t} + v_{t}$ is bounded for all $t$ uniformly in $x_{t}$ as
\begin{equation}
|| \textup{Cov}(w_{t} | x_{t}) || \leq \mathcal{E}_{T}:= \sigma^{2}_{s}(m, n) + 2 \frac{d S_{1}^{2}}{m^{2} \epsilon_{0}^{2}}, \label{eq: bound on the covariance of the overall noise}
\end{equation}
where $\epsilon_{0}$ is given in \eqref{eq: required stoc sigma} and $\sigma^{2}_{s}(m, n)$ is an upper bound on the norm of the covariance  of the error due to subsampling given by
\begin{equation} \label{eq: bound on the subsampling error}
\sigma^{2}_{s}(m, n) = \frac{S_{1}^{2}}{4}  \frac{1}{m} \frac{n-m}{n - 1}.
\end{equation}
\end{prop}
Note that $\mathcal{E}_{T}$ depends on the total number of iterations $T$ through $\epsilon_{0}$, hence the subscript. 

Before going into the detailed technical analysis, we find it useful to provide a sketch of it. Our purpose is to find an upper bound for the expected sub-optimality
$\mathbb{E}[F(x_{t}) - F^{\ast}]$ where $x^{\ast}$ is the optimal solution of \eqref{eq: main_prob} and $F^{\ast}:=F(x^{\ast})$ is the minimum value of $F$. The upper bound we will prove is of
the form
\[
  \mathbb{E}[F(x_{t}) - F^{\ast}] \leq \rho^{2t} \psi_{0} +
  \mathcal{E}_{T} R, \quad 0 \leq t \leq T,
 \]
 for some rate $\rho$, a non-negative $\psi_{0}$ that is related to the initial point $x_{0}$, and a non-negative $R$. As we will show soon, this bound in the DP setting has interesting aspects: Note that, as an issue unique to the differential privacy context, \emph{the term $\mathcal{E}_{T}$ increases with the total number of iterations, $T$}. This is because for fixed privacy level $\epsilon$, as $T$ increases $\epsilon_0$ defined in  \eqref{eq: required stoc sigma} decreases. Hence, increasing the number of iterations $T$ makes the first term $\rho^{2T} \psi_0$ smaller, however it leads to an increase in the second term $\mathcal{E}_T R$. This makes the analysis of DP-SHB fundamentally different compared to the analysis of the standard SHB in the stochastic optimization literature (see e.g. \cite{can2019accelerated, gadat2018stochastic, flammarion2015averaging}), where the second term is scaled with the fixed noise variance parameter that does not change with the number of iterations.

For analysis purposes, we define $\bar{F}: \mathbb{R}^{2d} \mapsto \mathbb{R}$ such that for $\xi_{t} = \begin{bmatrix} x_{t}^{\top} & x_{t-1}^{\top} \end{bmatrix}^{\top}$, we have $\bar{F}(\xi_{t}) = F(x_{t})$. Also, for a $2 \times 2$ symmetric positive-definite matrix $P$ and a positive scalar $c$, we set the Lyapunov function
\[
V_{P, c}(\xi) = V_{P}(\xi) + c(\bar{F}(\xi) - F^{\ast})
\]
with $V_{P}(\xi) = (\xi - \xi^{\ast})^{\top} [P \otimes I_{d}] (\xi - \xi^{\ast})$. The following proposition, which is constructed in a similar vein as Proposition 4.6 in \cite{aybat2018robust}, allows us to obtain expected sub-optimality bounds depending on the parameters $\alpha$ and $\beta$ as well as the noise level $\mathcal{E}_{T}$ and a convergence rate $\rho$.
% provided that there exists a $\tilde{P}$ matrix satisfying the $3\times 3$ matrix inequality \eqref{eq: mat_ineq}. 
A proof is given in Appendix \ref{app:Proof_2ndWay}.
\begin{prop} \label{prop: bd_HB} 
Given $F \in \mathcal{S}_{\mu, L}(\mathbb{R}^d)$, consider running DP-SHB algorithm with constant parameters $\alpha$ and $\beta$ for $T$ iterations and with $b_{t}(x_{t})$ in Proposition \ref{prop: algstocDP} so that $\epsilon$-differential privacy is satisfied. Suppose that Assumption \ref{asmp: Bounded L1 sensitivity} holds and there exists $\rho \in (0,1)$, a $2 \times 2 $ positive semi-definite symmetric matrix $P$, and constants $c_{0}, c \geq 0$ such that
\begin{equation} \label{eq: mat_ineq} 
c_0 X_{0} + c [ X_{1} + (1-\rho^2) X_{2} ] \succeq \Phi(A, B, P, \rho),
\end{equation}
where
\[
X_0 = \begin{bmatrix} 2\mu L C^{\top} C& -(\mu+L) C^{\top} \\ -(\mu+L) C & 2I_d \end{bmatrix},  \quad \Phi( A, B, P, \rho) = \begin{bmatrix} A^{\top} P A - \rho^{2} P  & A^{\top} P B \\ B^{\top} P A & B^{\top} P B \end{bmatrix},
\]
the matrices $A, B, C$ are as in \eqref{eq: DP-SHB system matrices}, and
\[
X_{1} = \frac{1}{2} \begin{bmatrix} -L \beta^{2} & L \beta^{2} & -(1-L\alpha) \beta \\ L \beta^{2} & -L \beta^{2} & (1-L\alpha) \beta \\ -(1-L\alpha) \beta & (1-L\alpha) \beta & \alpha(2-L\alpha)
\end{bmatrix}, \quad X_{2} = \frac{1}{2} \left[ {\begin{array}{ccc}
   \mu & 0 & -1 \\
   0 & 0 & 0 \\
  -1 & 0 & 0
  \end{array} } \right].
\]
Then, for all $0 \leq t \leq T$, we obtain
\begin{align}\label{eq:shb-perf-upper-bound}
 \mathbb{E}[F(x_{t}) - F^{\ast}] & \leq  \rho^{2t} \frac{1}{c}V_{P, c}(\xi_0) + \frac{1-\rho^{2t}}{1-\rho^2} \frac{L d \alpha^{2}}{2} \mathcal{E}_{T} \left( 1 + \frac{2 P_{12}^{2}}{P_{22} c L+ 2 |P| } \right), % R(P, c, L),
\end{align}
where $\mathcal{E}_{T}$ is defined in \eqref{eq: bound on the covariance of the overall noise}, $|P|$ denotes the determinant of $P$ and we have the convention $0/0 = 0$ for the last factor.
\end{prop}
As distinct from the approach in \cite{aybat2018robust}, which is developed for Nesterov's accelerated gradient method, the bound in \eqref{eq:shb-perf-upper-bound} is constructed by adapting the results for the deterministic HB \cite{hu2017dissipativity} to the stochastic setting. We also note that the matrix inequality \eqref{eq: mat_ineq} is $3\times 3$ and can be solved numerically for $\rho$ and $P$ in practice by a simple grid search over the rate $\rho$ and entries of the $2\times 2$ matrix $P$ (see, e.g., \cite{hu2017dissipativity, lessard2016analysis, can2019accelerated}). Therefore, the right-hand side of \eqref{eq:shb-perf-upper-bound} that provides performance bounds can be computed numerically in practice.

\subsection{Analysis of quadratic objective function case} \label{sec: Analysis of quadratic objective function case}

In this section, we will present explicit bounds for a quadratic objective function in order to provide more insight into the interplay between $\alpha$, $\beta$, and the number of iterations $T$. We consider the following quadratic function 
\begin{equation} \label{eq: quadratic objective}
F(x) = \frac{1}{2}x^{\top} Q x+a^{\top} x+b, 
\end{equation}
where $Q \in \RR^{d \times d}$ is symmetric positive definite, $a \in \RR^d$ a column vector and $b \in \RR$ is a scalar. For such a strongly convex quadratic objective function, an exact bound for the objective error can be presented.

To put it in a differential privacy context, we can assume that the parameters of $F(x)$ depend on some data $Y = \{y_{1}, \ldots, y_{n}\}$. For example, $F$ is a sum of functions $f(\cdot; y_{i})$ that are quadratic in $x$ (hence $F$ itself is quadratic in $x$), and the coefficients of the quadratic expression for each $f(\cdot; y_{i})$ depend on $y_{i}$. We will assume that, the $L_{1}$ sensitivity of $F$ is such that the required DP noise satisfies $\mathbb{E}(\eta_{t} \eta_{t}^{\top}) = \sigma_{T}^{2} I_{d}$ for some $\sigma_{T}^{2} > 0$. For simplicity, we assume that no subsampling is performed, i.e., $v_{t} = 0$.

The optimal values for HB in the non-noisy setting has been given in \cite{polyak1987introduction} as $\alpha_{\textup{HB}} = 4/(\sqrt{\mu}+\sqrt{L})^2$ and $\beta_{\textup{HB}} = (\sqrt{\kappa}-1)^{2}/(\sqrt{\kappa}+1)^{2}$ where $\kappa := L/\mu$. 
% (the related theorem is quoted in Appendix). 
However, those ``optimal'' values may not be the best selection for $\alpha$ and $\beta$ for DP-SHB. There are two reasons for this: First, due to privacy concerns, noise is inevitable in DP-SHB. Presence of noise shows as a second additive term in the bound for the error. This second term is affected by the selection of $\alpha$. Second, the amount of privacy preserving noise increases with the total number of iterations. In general, the error bound is a sum of two terms. The first of these two terms decreases with the convergence rate $\rho$ of the algorithm and the second term is due to privacy preserving noise. It will be shown that $\alpha$ and $\beta$ have an influence on both the convergence rate and the multiplicative constant of the additive error due to noise. We will additionally see that a selection of $\alpha, \beta$ pair that improves the rate also increases the additive error term due to the presence of privacy preserving noise. Therefore, we can talk about a trade-off between the convergence rate and the additive noise term in our performance bounds, which is adjusted by the parameters $\alpha$ and $\beta$. In that respect, the ``optimal'' $\alpha$ and $\beta$ in the non-noisy setting is typically not the best choice of $\alpha$ and $\beta$ in the DP setting.

By adapting \cite[Thm 12]{can2019accelerated}, given for parameter choices $\alpha_{\textup{HB}}, \beta_{\textup{HB}}$, we present our result for the error bound given by any pair $\alpha, \beta$. A proof is given in Appendix \ref{app:C}.
\begin{thm} \label{thm: opt_quadraticHB} 
Let $F \in \mathcal{S}_{\mu, L}(\mathbb{R}^d)$ be a quadratic function given in \eqref{eq: quadratic objective}. Consider the iterates $\{x_{t} \}_{0 \leq t \leq T}$ of the DP-SHB method, which is run for $T$ iterations with noisy gradients $\nabla F(x_{t}) + w_{t}$ where $\mathbb{E}(w_t | x_t) = 0$ and $\mathbb{E}(w_t w_t^{\top} | x_t) \preceq \sigma_{T}^{2} I$ for some positive constant $\sigma_{T}^{2} > 0$. If DP-SHB is run with parameters $(\alpha, \beta)$, then
\begin{equation} \label{eq: error bound for quadratic function_main}
\mathbb{E}[F(x_{t})] - F(x^{\ast}) \leq V(\xi_{0}) C_{t}^{2} \rho^{2t} + Lm(\alpha, \beta),
\end{equation}
where 
\[
m(\alpha, \beta) = \frac{\sigma_{T}^{2}}{2} \sum_{i=1}^{d}\frac{2 \alpha (1 + \beta)}{(1-\beta)\lambda_{i}(2 + 2\beta - \alpha \lambda_{i})}
\]
with $\lambda_{i}$'s being the eigenvalues of $Q$. In \eqref{eq: error bound for quadratic function_main}, we have
\begin{equation} \label{eq: convergence rate of DP-SHB} 
\rho = \max \{ | a_{\mu,+}|, |a_{\mu,-}|, |a_{L,+}|, |a_{L,-} |\},
\end{equation}
where
\[
a_{\lambda, \pm} = \frac{(1+\beta)(1-\alpha\lambda) \pm\sqrt{(1+\beta)^{2}(1-\alpha\lambda)^{2}-4\beta(1-\alpha\lambda)}}{2}
\]
and $V(\xi_{0})$ is given by 
\[
V(\xi_{0}) = \mathbb{E}[\norm{(\xi_{0} - \xi^{\ast})(\xi_{0} - \xi^{\ast})^{\top}}]+ \frac{ \sigma_{T}^{2} \alpha^{2}}{1-\rho^{2}}
\]
with $C_{t} = \mathcal{O}(t)$ being a sequence of scalar coefficients, provided that $\rho < 1$.
\end{thm}
Note that in Theorem \ref{thm: opt_quadraticHB} we considered the case with uncorrelated and bounded noise variance, which generalizes over the independent noise setting. To the best of our knowledge, such a result has not been shown before in the literature. 

\paragraph{\textbf{Numerical demonstration:}} Here, we illustrate the effect of algorithm parameters over the error bound given in Theorem \ref{thm: opt_quadraticHB}. The dimension of the objective function is taken $d = 2$, and $Q$ is chosen as the $2 \times 2$ diagonal matrix 
% \[
% Q = \begin{bmatrix} \mu & 0 \\ 0 & L \end{bmatrix}
% \]
with $\mu = 0.5$ and $L = 1$ on its diagonal, so that its eigenvalues are $\mu$ and $L$. 

We take $C_{t} = t$ for simplicity of the presentation.\footnote{$C_t$ is a constant multiple of $t$ but the constant in front of $t$ would not change the qualitative behavior of the plots, only shifting the graphs by a constant factor in the logarithmic scale.}  With fixed stepsizes $\alpha \leq 1/L$, the convergence rate $\rho$ in \eqref{eq: convergence rate of DP-SHB} versus $\beta$ is plotted in Figure \ref{fig: quadraticHB_bound} for several values of $\alpha$. As for the noise variance, we considered $\sigma_{T}^{2} = (T c_{w})^{2}$ to represent increasing noise variance in the total number of iterations. We repeated our experiments for two different values of $c_{w}$, namely for $c_{w} = 10^{-4}$ (representing a less noisy, hence less private scenario), and $c_{w} = 10^{-2}$ (representing a more noisy, hence more private scenario). We observe that the ``optimal'' $\beta$ value in terms of convergence rate $\rho$ (which is indicated at the bottom row of Figure \ref{fig: quadraticHB_bound}) shows a reliable performance.

\begin{figure}[ht]
\begin{center}
\includegraphics[scale=0.75]{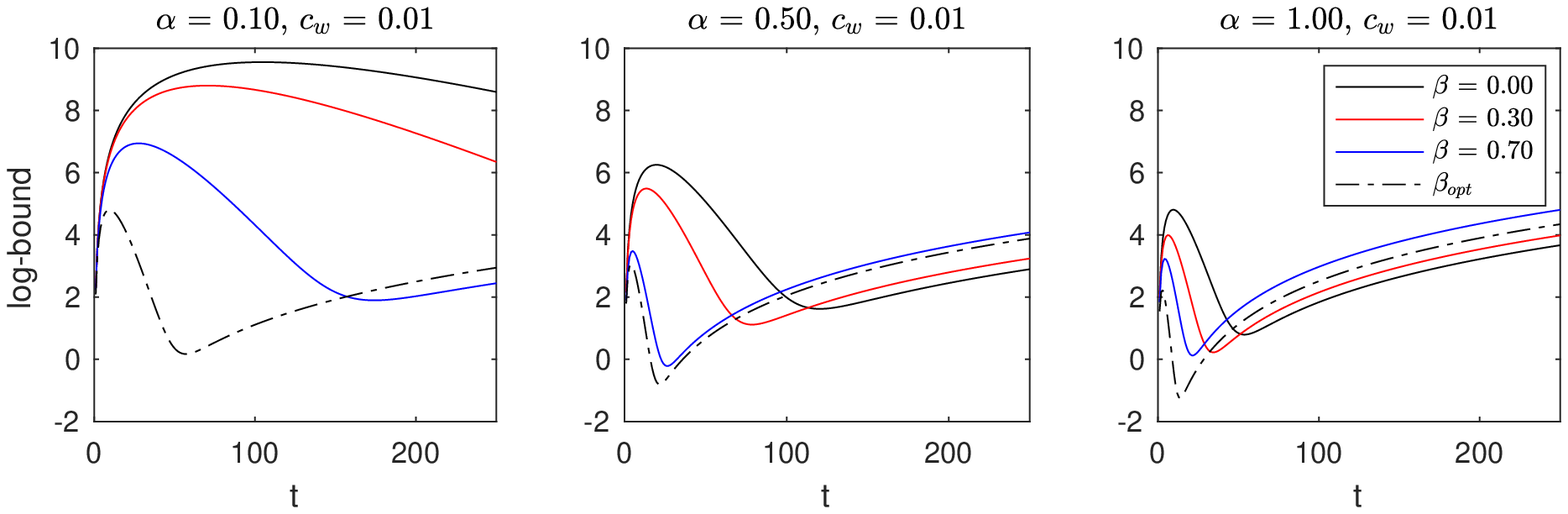} \\
\includegraphics[scale=0.75]{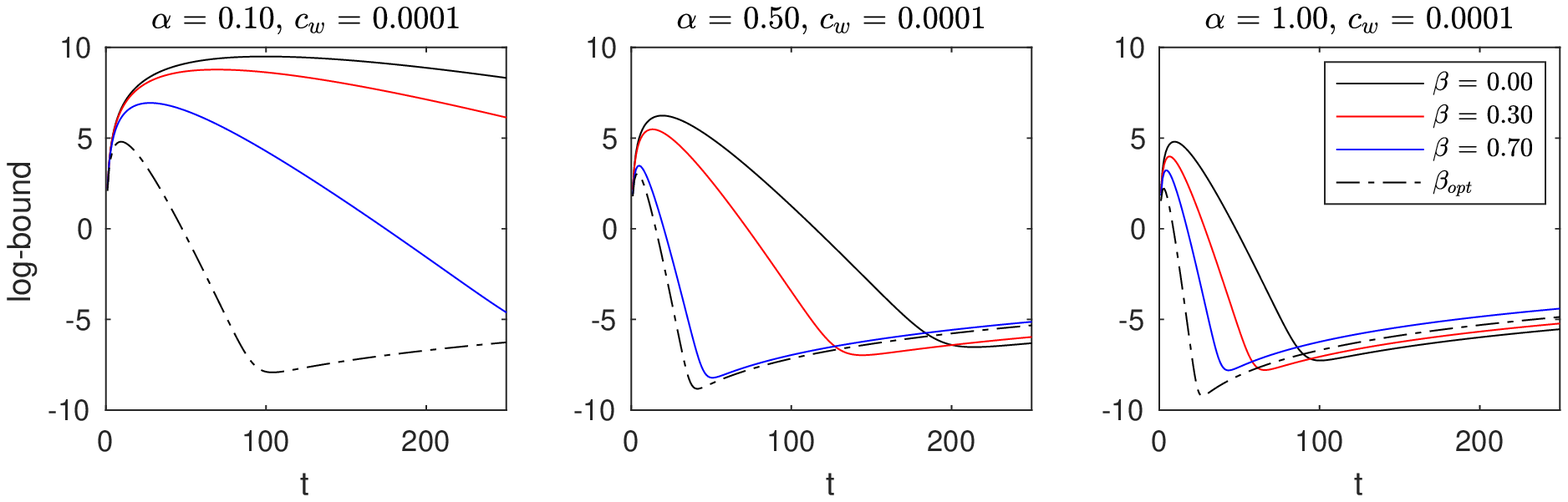} \\
\includegraphics[scale=0.75]{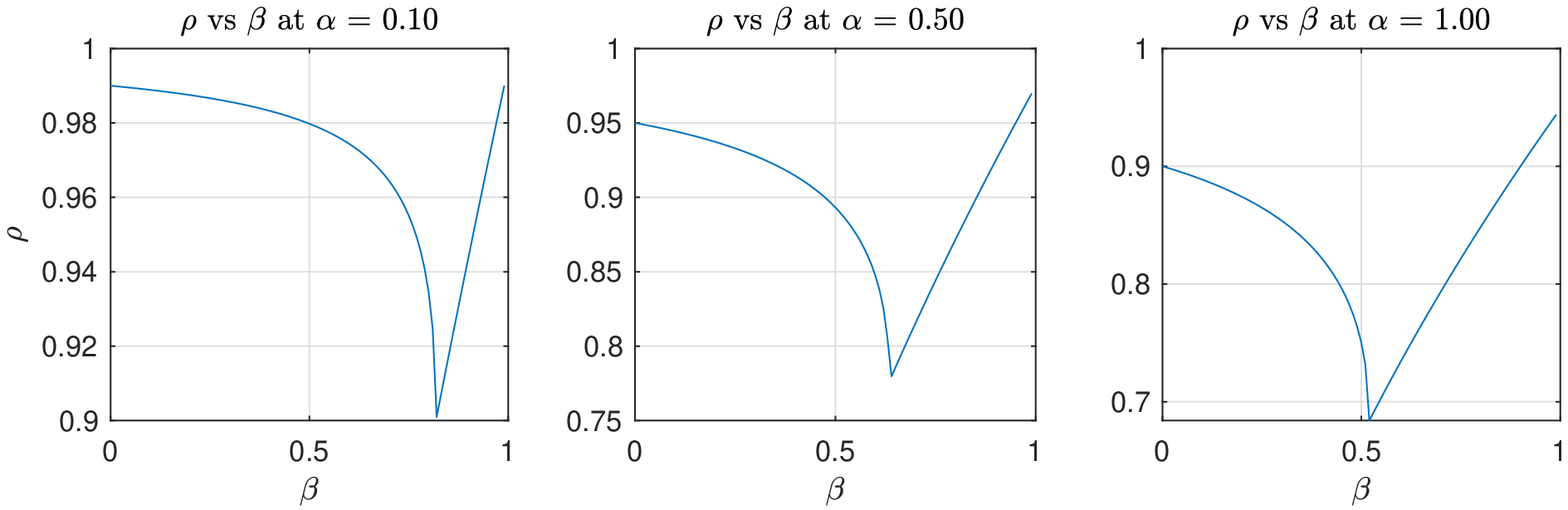}
\end{center}
\caption{DP-SHB performance for the quadratic objective function case}
\label{fig: quadraticHB_bound}
\end{figure}

\section{Differentially Private Accelerated Algorithms} \label{sec: Differentially Private Accelerated Algorithms}

In this section, we will investigate NAG in a differential privacy setting, and propose two ways to tailor it for improved performance under differential privacy.

In the following discussion, we will assume that Assumption \ref{asmp: Bounded L1 sensitivity} on the existence of an upper bound $S_{1}$ on the sensitivity holds, like in the previous section. Furthermore, we will assume that the upper bound $S_{1}$ is considered while determining the parameter $b_{t}$ of the privacy preserving Laplace, so that $b_{t}$ is independent from the current state. Using a state-independent sensitivity to determine the Laplace parameter is not uncommon, especially when it is hard to identify $S_{1}(x)$ for all $x$. An example to this case can be found in Section \ref{sec: Experimental results}, in particular the sensitivity bound in \eqref{eq: sensitivity_logreg} for the logistic regression model, which is independent of the state $x$.
     
Recall the NAG update in \eqref{eq: NAG Basic}. A straightforward differentially private version of NAG would be obtained by cluttering the gradient with the privacy-preserving noise, just as in the DP-SHB algorithm. The corresponding change in NAG would be
\begin{align}\label{eq: DP-NAG Basic}
\begin{split}
  & x_{t+1}  = z_{t} - \alpha (\nabla F(z_{t}) + \eta_{t}), \\
  & z_{t} = (1+\beta) x_{t} - \beta x_{t-1},
\end{split}
\end{align}
where $\eta_{t, i} \overset{\textup{i.i.d.}}{\sim} \textup{Laplace}(b)$ for $i = 1, \ldots, d$, and $b = \frac{S_{1}}{n \epsilon_{0}}$ with $\epsilon_{0}$ is given in \eqref{eq: epsilon per iteration}. The resulting algorithm will be referred to as DP-NAG.

In DP-NAG, the stepsize $\alpha$ (hence $\beta$) and the DP noise parameter $b$ are taken constant. That begs the question whether the performance of DP-NAG could be improved, if we let $b$ and $\alpha$ depend on $t$, the iteration number. We propose two methods to improve the performance of DP-NAG while preserving the same level of privacy. The first method seeks to improve the algorithm by making the DP variance parameter $b$ dependent on the iteration number, whereas the second considers varying $\alpha$ (hence $\beta$) with iterations.

\subsection{NAG with optimized DP variance}\label{subsec:DP-NAG} We first present an error bound for NAG that uses noisy gradients. Let $E_{t} = \mathbb{E}(F(x_{t})) - F^{\ast}$. The following theorem is adapted from \cite[Theorem 2.3]{aybat2019masg}.

\begin{thm}
Let $F \in \mathcal{S}_{\mu, L}(\mathbb{R}^d)$ and suppose that Assumption \ref{asmp: Bounded L1 sensitivity} holds. Consider a stochastic version of the NAG algorithm that runs with a stepsize $\alpha \leq 1/L$ and the momentum parameter $\beta = (1-\sqrt{\alpha \mu})/(1+\sqrt{\alpha \mu})$ and uses noisy gradients $\widetilde{\nabla F(z_{t})} = \nabla F_{B_{t}}(z_{t}) + \eta_{t}$ for $t \geq 0$ as in \eqref{eq: noisy gradient vector} with a subsampling size $m \leq n$ and $\eta_{t, i} \overset{\textup{i.i.d.}}{\sim} \textup{Laplace}(b_{t})$ for all $i = 1, \ldots, d$. Then, for any $t \geq 1$, we have
\begin{equation} \label{eq: backward recursion for error} 
E_{t} \leq (1 - \sqrt{\mu \alpha}) E_{t-1} + \alpha (1 + \alpha L) \left( b_{t}^{2}  d + \sigma^{2}_{s}(m, n)/2 \right).
\end{equation}
\end{thm}
Note that in \eqref{eq: backward recursion for error}, the term $b_{t}^{2} + \frac{S_{1}^{2}}{m} \frac{n-m}{n-1} $ is an upper bound on the norm of the covariance of the gradient estimator, and it simplifies to $b_{t}^{2}$ when $m = n$, i.e., without subsampling. By starting the recursion in \eqref{eq: backward recursion for error} at the last iteration $t = T$ and recursing backward until $t = 0$, we end up with
\[
E_{T} \leq (1 - \sqrt{\mu \alpha})^{T} E_{0} + \sum_{t = 1}^{T} (1 - \sqrt{\mu \alpha})^{T - t} \alpha (1 + \alpha L) \left( b_{t}^{2} d  + \sigma^{2}_{s}(m, n)/2 \right).
\]
It will prove useful later to express the error of NAG generically as
\begin{equation} \label{eq: generic error for NAG algorithms}
E_{T} \leq a_{T, 0} E_{0} + \sum_{t = 1}^{T} a_{T, t} \left( b_{t}^{2} d  + \sigma^{2}_{s}(m, n)/2 \right).
\end{equation}
The $a_{T, t}$ in \eqref{eq: generic error for NAG algorithms} can
be identified as
\[
a_{T, t} = \begin{cases} (1 - \sqrt{\mu \alpha})^{T}, & t = 0; \\
(1 - \sqrt{\mu \alpha})^{T - t} \alpha (1 + \alpha L), & t = 1, \ldots, T. \end{cases}
\]

In the DP framework, we have control on the noise parameters $b_{t}$, with a constraint due to our privacy budget $\epsilon$. Suppose that we are committed to run the algorithm for a total of $T$ iterations. When $b_{t}$ is used, the privacy leak at iteration $t$ becomes $\epsilon_{t} = \varepsilon(S_{1}, b_{t}, n, m)$. Given a desired privacy level $\epsilon$, we have the constraint $\sum_{t = 1}^{T} \epsilon_{t} = \epsilon$, by Theorem \ref{thm: composition}.
%\[
%\sum_{t = 1}^{T} \frac{S_{1}}{n b_{t}} = \epsilon.
%\]
Therefore, one question is, with fixed $m$ and $T$, how we should arrange $b_{t}$ so that the bound in \eqref{eq: generic error for NAG algorithms} is optimized. Factoring our privacy budget into the scene, we have the following constrained optimization problem.
\begin{equation} \label{eq: optimization problem for NAG}
\min_{b_{1}, \ldots b_{T}} \quad \sum_{t = 1}^{T} a_{T, t} b_{t}^{2}, \quad \text{ subject to } \sum_{t = 1}^{T} \varepsilon(S_{1}, b_{t}, n, m) = \epsilon.
\end{equation}
For general $m \neq n$, the constrained optimization problem is analytically intractable and needs a numerical solution. This is due to the non-linearity in $\varepsilon(S_{1}, b_{t}, n, m)$. However, for the special case of $m = n$ (no subsampling), the constraint in \eqref{eq: optimization problem for NAG} simplifies to $\sum_{t = 1}^{T} S_{1}/nb_{t} = \epsilon$, allowing for the following tractable result. (A proof is given in Appendix \ref{sec: Proof of optimal b}.)
\begin{prop} \label{prop: optimal sigma} 
When $m = n$, the optimization problem in \eqref{eq: optimization problem for NAG} is solved by
\begin{equation} \label{eq: optimum variances - generic} 
b_{t} = \frac{\sum_{j = 1}^{T} a_{T, j}^{1/3} }{a_{T, t}^{1/3}} \frac{S_{1}}{n \epsilon}, \quad t = 1, \dots, T.
\end{equation}
\end{prop}
We can express the solution \eqref{eq: optimum variances - generic} also in terms of the privacy leak at iteration $t$ as:
\[
\epsilon_{t} = \frac{a_{T, t}^{1/3} }{\sum_{j = 1}^{T} a_{T, j}^{1/3}} \epsilon, \quad t = 1, \dots, T.
\]
Since $a_{T, t}$ is decreasing in $t$, the solution \eqref{eq: optimum variances - generic} suggests that the variance should start high and then should be decreased. This means that the privacy budget should be distributed to the iterations in an unevenly way. A larger part of the privacy budget should be spent for later rather than for early iterations.

\begin{rem} \label{rem: optimize w.r.t T}
The solution in \eqref{eq: optimum variances - generic} for $m = n$ yields the optimum bound
\begin{equation} \label{eq: optimized bound}
E_{T} \leq a_{T, 0} E_{0} + \frac{d S_{1}^{2}}{n^{2} \epsilon^{2}} \left(\sum_{j = 1}^{T} a_{T, j}^{1/3} \right)^{3},
\end{equation}
which could further be optimized with respect to the number of iterations, provided that one has an accurate guess on the initial error $E_{0}$. We note that increasing the number of iterations may degrade the performance in the DP context, since the required noise per iteration increases unlike the deterministic setting where one may improve the performance monotonically as the number of iterations grows.
\end{rem}

\begin{rem} \label{rem: optimize under subsampling}
Although the result in Proposition \eqref{prop: optimal sigma} is valid for no subsampling, it can be used as a guide for arranging $b_{t}$'s even under subsampling. Note that for values of $m$, $n$, $S$, and $b$ such that  $m \ll n$ and $S/bm \ll 1$, we have $\varepsilon(S, b, m, n) \approx S/bn$, owing to the approximation $e^{z} \approx 1 + z$ for $z \ll 1$.
\end{rem}

\subsection{Multi-stage NAG}
\label{se: Multi-stage NAG}
An alternative for improving the performance of NAG is to make the stepsize vary with iterations. In fact, the MASG algorithm of \cite{aybat2019masg} has been proposed with that motivation. The authors prove that MASG achieves optimal rate both in deterministic and stochastic versions.

In this paper, we present a DP-MASG, a differentially private version of MASG introduced in \cite{aybat2019masg}. In order to study and improve the error behavior of the algorithm, an explicit bound for the objective error that accommodates iteration dependent noise variance parameter $b_{t}$ is presented. We demonstrate that the approach of dividing noise into iterations can be applied to MASG as well.

The original algorithm MASG is a multistage accelerated algorithm which uses Nesterov's accelerated gradient method with noisy full gradient.  The total iterations $T$ are divided into $K$ stages, with stage lengths $n_{k}$, and for each stage a different stepsize $\alpha^{(k)}$ is used. For the optimal convergence rate, the stage lengths and the corresponding stepsizes are recommended in \cite{aybat2019masg} as
\begin{equation} \label{eq: MAG stepsize}
n_1 \geq 1,  \quad  \alpha^{(1)} = \frac{1}{L}, \quad n_k = 2^{k} \left \lceil \sqrt{\kappa}\ln(2^{p+2}) \right \rceil , \quad \alpha^{(k)} = \frac{1}{2^{2k}L}, \quad k \geq 2,
\end{equation}
where $p \geq 1$. 

The MASG algorithm can easily be modified to be differential private by adding a Laplace noise to the gradient as in \eqref{eq: bound on  the covariance of the overall noise}. We will refer to the resulting algorithm as DP-MASG. The selections in \eqref{eq: MAG stepsize} for the stage lengths and the stepsizes were designed for constant noise variance per iteration. In the following, we will instead propose a new version that uses a variable noise variance parameter $b_{t}$ at iteration $t$, which can improve performance. The main idea is to rely on Proposition \ref{prop: optimal sigma} to optimize over $b_{t}$'s with the privacy budget constraint.

In order to study how the privacy noise can be optimally distributed to the iterations of DP-MASG, we provide an explicit bound that not only accommodates iteration-dependent noise variance, but also is in the same form as \eqref{eq: generic error for NAG algorithms} so that the noise variances can be optimized to minimize the bound.  For MASG, stepsizes change across stages, therefore, the recursion in \eqref{eq: backward recursion for error} cannot be applied for all iterations. Instead, by Lemma 3.3 of \cite{aybat2019masg}, we have a factor of two that appears when the algorithm transitions from one stage to the next. This leads to the following theorem.
\begin{thm} \label{thm: DP-MASG} Let $F \in \mathcal{S}_{\mu, L}(\RR^d)$. Consider the DP-MASG algorithm with stage lengths  $n_{k}$ and step-sizes during those stages $\alpha^{(k)}$ given as in \eqref{eq: MAG stepsize}, and with noisy gradients $\nabla f(x_{t}) + \eta_{t}$, where $\eta_{t, i} \overset{\textup{i.i.d.}}{\sim} \textup{Laplace}(b_{t})$ for $i = 1, \ldots, d$. Then,
\begin{equation} \label{eq: upper bound for error} 
\begin{split}
E_{T} =& \left[ 2^{s_{T}-s_{0}} \prod_{i = 1}^{T} (1 - \sqrt{\mu \alpha^{(s_{i})}}) \right] E_{0} \\
 & \quad\quad\quad\quad + \sum_{t = 1}^{T} 2^{s_{T}-s_{t}} \left[ \prod_{i = t+1}^{T} (1 - \sqrt{\mu \alpha^{(s_{i})}}) \right] \alpha^{(s_{t})} (1 + \alpha^{(s_{t})} L)\left( b_{t}^{2} d +  \sigma^{2}_{s}(m, n)/2 \right),
\end{split}
\end{equation}
where $s_{i}$ is the stage that contains iteration $i$, provided that $\alpha^{(k)} \leq 1/L$ for all $k \geq 1$.
\end{thm}
Observing that the bound in \eqref{eq: upper bound for error} is in the same form as \eqref{eq: generic error for NAG algorithms}, $b_{t}$ can be optimized as in \eqref{eq: optimization problem for NAG} but with $a_{T, t}$ indicated by \eqref{eq: upper bound for error} as
\[
a_{T, t} = 2^{s_{T}-s_{t}} \left[ \prod_{i = t+1}^{T} (1 - \sqrt{\mu \alpha^{(s_{i})}}) \right] \alpha^{(s_{t})} (1 + \alpha^{(s_{t})} L), \quad t = 1, \ldots, T.
\]
Once again, the optimal $b_{t}$'s when $m = n$ can be written in terms of $\epsilon$, $S_{1}$, and $a_{T, t}$'s as in \eqref{eq: optimum variances - generic}. To show the effect of algorithm parameters on noise variance, we plot the optimum $b_{t}$ values in Figure \ref{fig:optbt} for $\mu = 1$, $L = 20$,
$\kappa = 20$, $p = 1$, and $c_{1} = 1$, representing the
constant factor in front of the stepsize.
\begin{figure}[hbt]
\centerline{
\includegraphics[scale = 0.75]{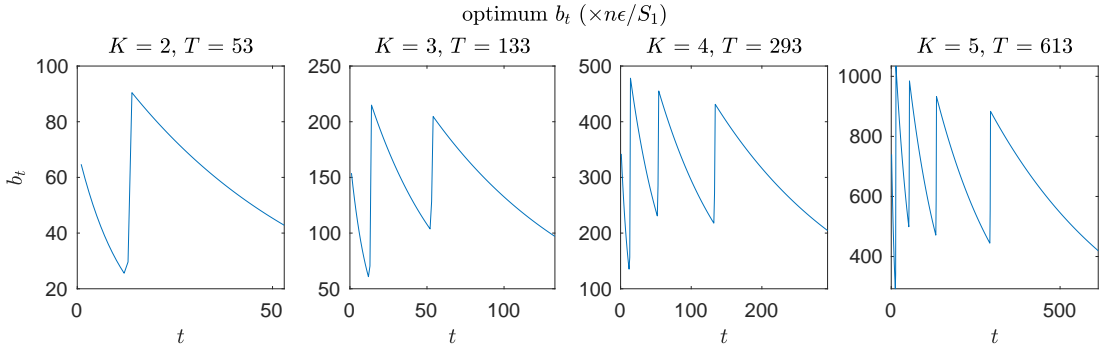}
}
\caption{Optimal $b_{t}$ values for the multi-stage NAG algorithm.}
\label{fig:optbt}
\end{figure}

\section{Experimental results} \label{sec: Experimental results}

Our experiments concern a regularized logistic regression problem.\footnote{The results are produced with the code at \url{https://github.com/sibirbil/DPAccGradMethods}} The model has observations $y_{i} = (u_{i}, z_{i})$, $i = 1, \ldots, n$, where $u_{i} \in \mathcal{U} \subseteq \mathbb{R}^{d}$ is a vector of covariates and $z_{i} \in \{ -1, 1\}$ is a binary response whose conditional probability given $u_{i}$ depends on a parameter vector $x \in \mathbb{R}^{d}$ as follows:
\[
p(z_{i} | u_{i}, x) = \left[1 + e^{-z_{i} u_{i}^{\top} x}\right]^{-1}, \quad i = 1, \ldots, n.
\]
Since the probability distribution of $u_{i}$'s does not depend on $x$, the (regularized) maximum likelihood problem is defined as determining
\begin{align} \label{eq:problem_det} x^{\ast} = \arg\max_{x \in \RR^{d}} \frac{1}{n}\sum_{i = 1}^{n} f(x; u_{i}, z_{i}),
\end{align}
where $f(x; u_{i}, z_i):= \ln p(z_{i} | u_{i}, x) + \lambda \norm{x}^2$. One can verify that $S_{1}(x) = 2 \sup_{u \in \mathcal{U}} \| u \|_{1}$ for all $x \in \mathbb{R}^{d}$, upon observing that, for all $u, u' \in \mathcal{U}$, and $z, z' \in \{0,  1\}^{2}$, we have
\begin{align}
  \| \nabla f(x; u, z) - \nabla f(x; u', z') \|_{1} &= \left\| \frac{z u e^{z u^{\top} x}}{1 + e^{z u^{\top} x}} - \frac{z' u^{\prime} e^{z u^{\prime \top} x}}{1 + e^{z' u^{\prime \top}  x}}  \right\|_{1} \nonumber \\
& \leq  \| u \|_{1} \left\vert \frac{z e^{z u^{\top} x}}{1 + e^{z u^{\top} x}} \right\vert + \| u' \|_{1} \left\vert \frac{z'  e^{z u^{\prime \top} x}}{1 + e^{z' u^{\prime \top} x}}  \right\vert \nonumber  \\
& \leq \| u \|_{1} + \| u' \|_{1}. \label{eq: sensitivity_logreg}
\end{align}

For the experiments to follow, we use a synthetic data with $d = 20$ and $n = 10^{5}$ and the value of regularization parameter $\lambda$ is taken as $0.01$. The set $\mathcal{U}$ is taken as the set of all $d \times 1$ real-valued vectors with an $L_{1}$-norm less than equal to $20$. Hence, Assumption \ref{asmp: Bounded L1 sensitivity} holds for this example with $S_{1} = 2 \times 20$. We set $\mu = 2 \times \lambda$ and $L$ is estimated as the largest singular value of $\frac{1}{n}(U^{\top}U) + 2 \lambda I_{d}$, where $U$ is the $n \times d$ matrix with $u_{t}$ being its column $t$.

In our experiments, we compared six differentially private algorithms. The first four, DP-GD, DP-NAG, DP-MASG and DP-HB are the straightforward differentially private versions of GD, NAG, MASG and HB, respectively. The last two algorithms in the comparison are named DP-NAG-opt and DP-MASG-opt, who stand for the alterations of DP-NAG and DP-MAGS for which the privacy preserving noise is distributed to the iterations according to Proposition \ref{prop: optimal sigma}. 

The algorithms are compared across different values of $m$, $T$, and $c$, where $m$ is the subsampling size, $T$ is the number of iterations, and $c$ determines the step size as in $\alpha = c/L$. For DP-MASG and DP-MASG-opt,  the general stepsize formulation in \eqref{eq: MAG stepsize}, presented for the original versions, is preserved; however the stepsizes are scaled by $c$. We tried all the combinations $(m, T, c)$ of $m = 10^{3}, 10^{5}$, $T = 100, 200, 500,  1000$, and $c = 0.1, 1$. We fixed $\epsilon = 1$ throughout the whole experiments. 

For DP-NAG-opt and DP-MASG-opt, we also adjusted the given value of $T$ as follows: With an initial guess of $E_{0} = 10$, we computed the bound in \eqref{eq: optimized bound} for each $T' \leq T$, and we decided the number of iterations to be that $T'$ that gives the minimum bound. This procedure was detailed in Remark \ref{rem: optimize w.r.t T}.

Figures \ref{fig: comp_all_m_1000} and \ref{fig: comp_all_m_100000} show, respectively for $m = 100000$ (no subsampling) and for $m = 1000$, the performances of the algorithms for the tried values of $c$ and $T$. Each subfigure shows the log-difference between the objective function evaluated at the current iterate $F(x_{t})$, and the objective function evaluated at the optimum solution $F(x^{\ast})$. The optimum solution was found with a non-private NAG algorithm that is run for 1000 iterations and without subsampling. The plotted values are the averages from 20 independent runs for each combination of $(m, T, c)$. Trace plots of the iterates for the different values of $T$ are plotted together with different colors. Note that for some cases plots overlap, leading some colors invisible.

\begin{figure}[!ht]
\begin{center}
  \includegraphics[scale = 0.70]{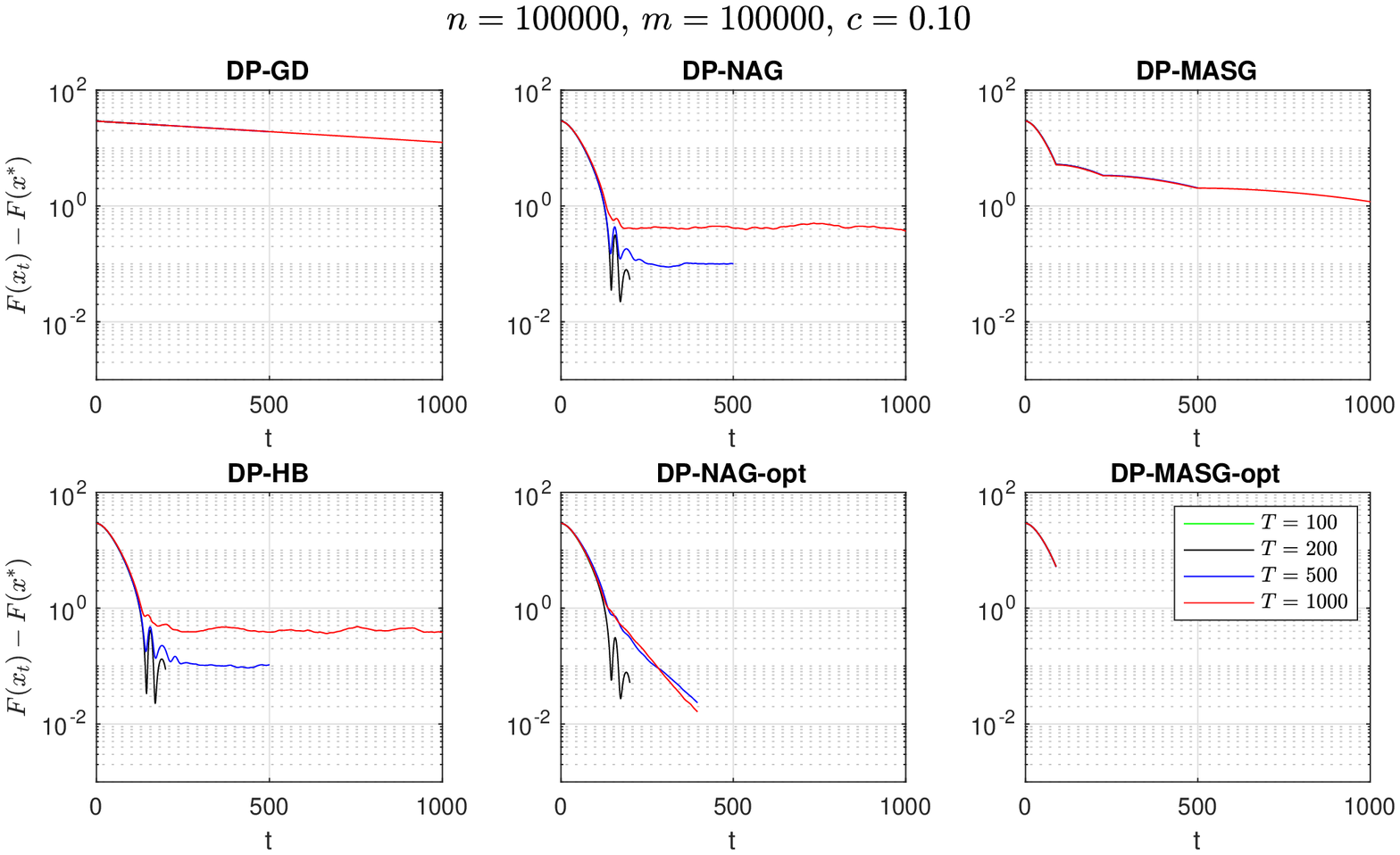} \\
  \includegraphics[scale = 0.70]{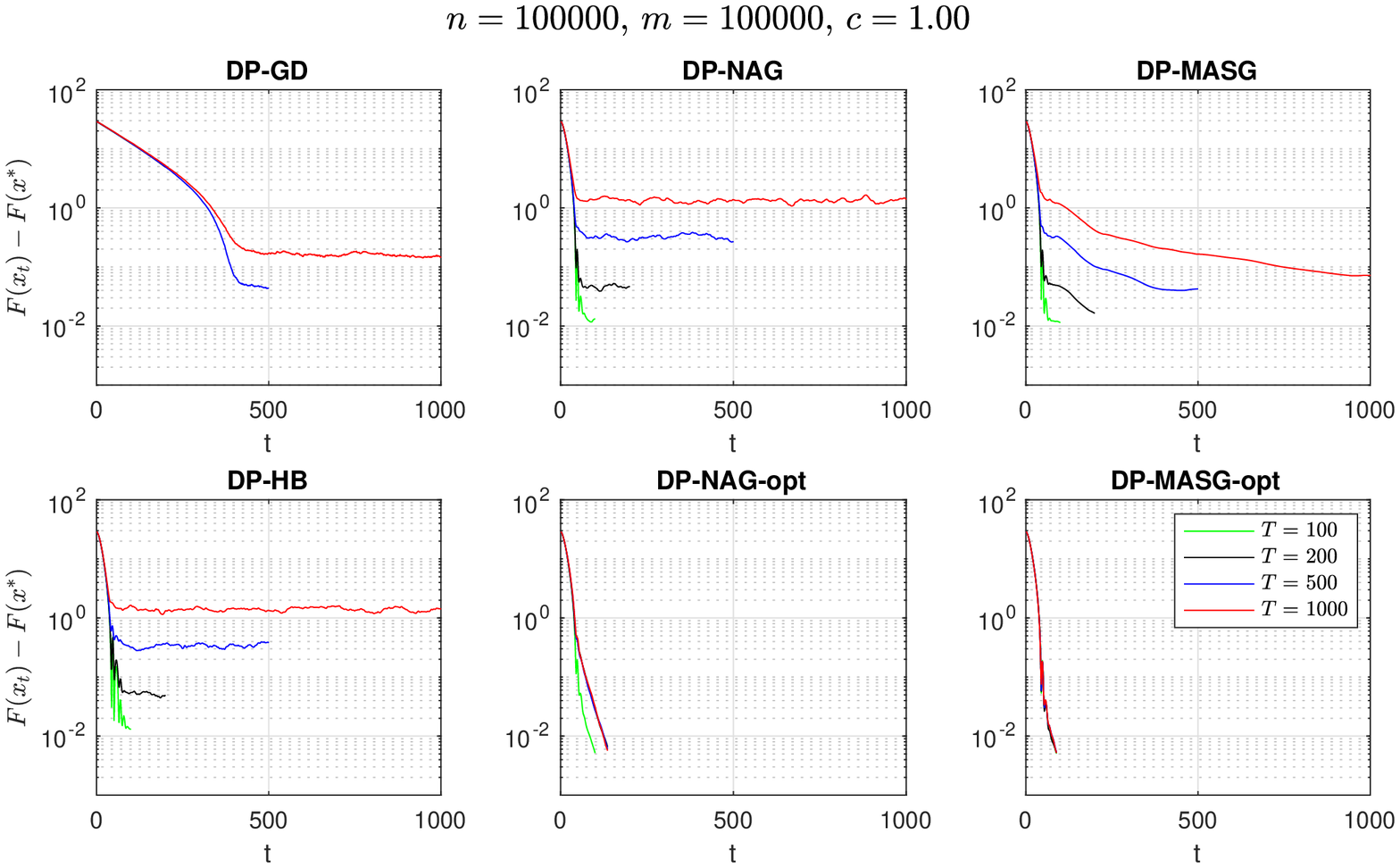}
\end{center}
\caption{Errors with various $T$ and $c$ and without subsampling. Top: $c = 0.1$, Bottom: $c = 1$.}
\label{fig: comp_all_m_1000}
\end{figure}
\begin{figure}[!ht]
\begin{center}
  \includegraphics[scale = 0.70]{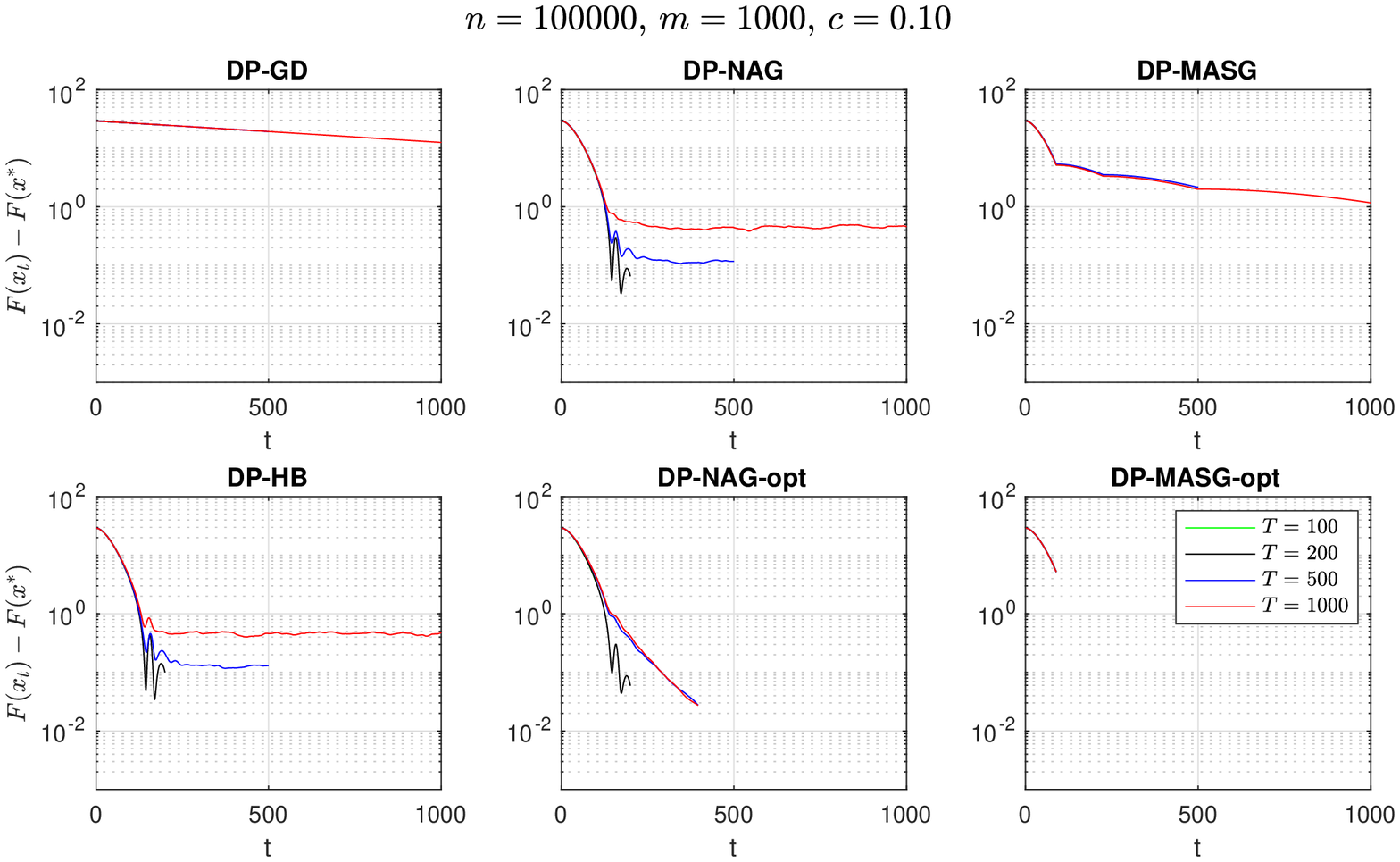} \\
  \includegraphics[scale = 0.70]{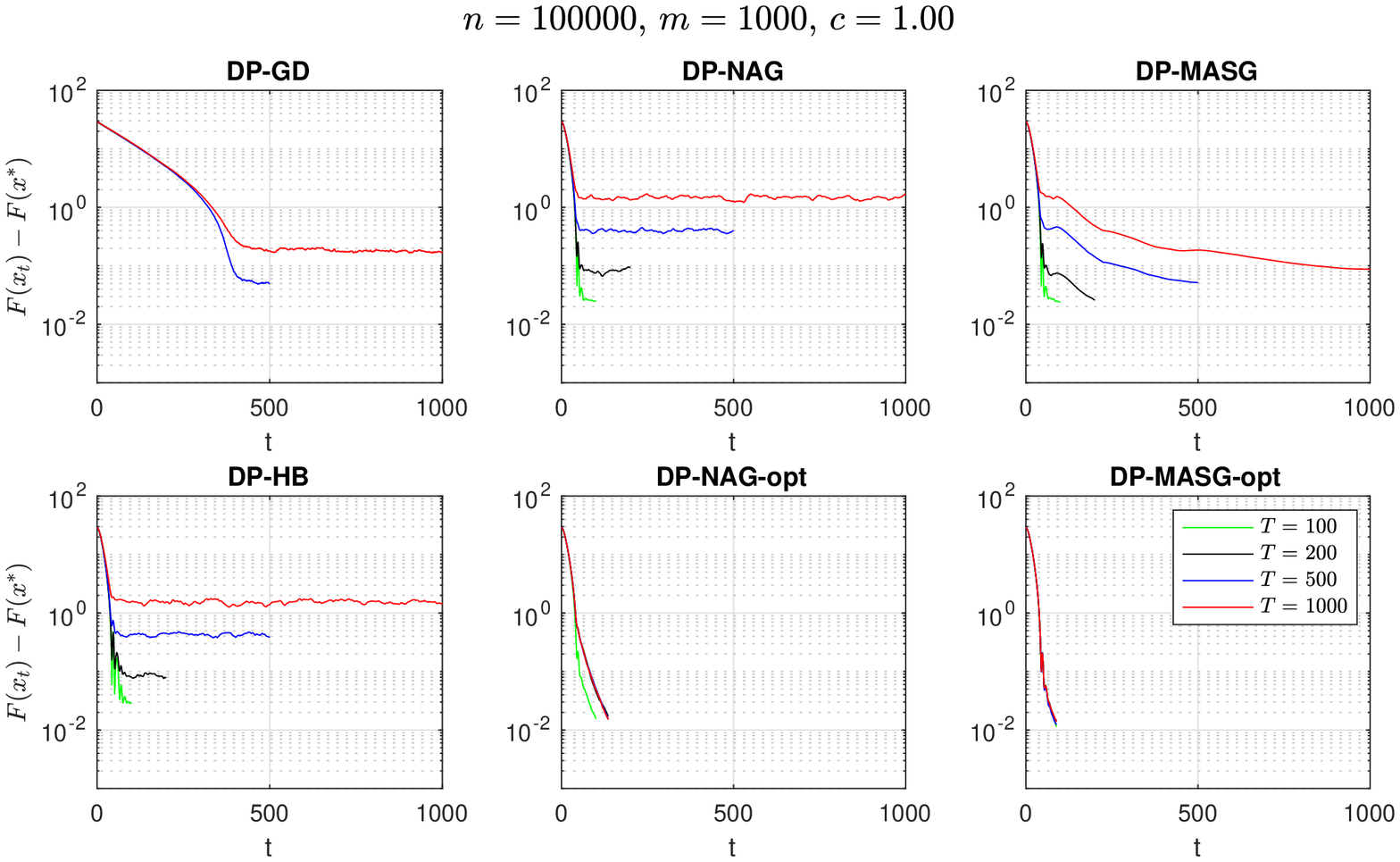}
\end{center}
\caption{Errors with various $T$ and $c$ and with $m = 10^{3}$. Top: $c = 0.1$, Bottom: $c = 1$.}
\label{fig: comp_all_m_100000}
\end{figure}

Comparing DP-GD against the accelerated algorithms, we observe that the accelerated algorithms, DP-HB, DP-NAG, and DP-NAG-opt, outperform DP-SGD. Furthermore, among the accelerated algorithms, we have the best results with DP-NAG-opt and DP-MASG-opt. The advantage of accelerating is more striking for $c = 0.1$, representing a too small value for the stepsize. While DP-DG is dramatically slow with a too small stepsize, the accelerated algorithms DP-HB, and DP-NAG, and DP-NAG-opt seem to suffer less from that ill choice for the stepsize. When $c = 1$, DP-GD recovers from slow convergence, however the accelerated algorithms are still able to beat it. Our observations hold both for $m = 100000$ and for $m = 1000$. The multistage algorithm DP-MASG is also prone to a small value for $c$; but it recovers dramatically when $c = 1$ as recommended in the original work \cite{aybat2019masg}.

In all instances, we can see the advantage of accelerated algorithms in the speed of convergence. However, when we compare the the error levels that the algorithms have reached for the same $T$, we see that sometimes DP-GD has a better performance over DP-NAG or DP-HB. See, for example, the lower half of Figure \ref{fig: comp_all_m_1000}, at $T = 1000$ (red line): while DP-GD converged more slowly than DP-HB and DP-NAG, it reached a smaller error level. However, if we conduct an \emph{overall} comparison between DP-GD and DP-HB in terms of their best performances among all the choices $T = 100, 200, 500, 1000$, we see that the best of DP-HB (at $T = 100$) outperforms the best of DP-GD (at $T = 1000$). This observation is repeated in our experiments and is suggestive of a general recommendation: The accelerated algorithms can promise faster convergence when used with a small number of iterations.

This general recommendation about the selection of $T$ is further supported by the traces belonging to the DP-NAG-opt and DP-MASG-opt (when $c = 1$), where $T$ is re-adjusted according to \eqref{eq: optimized bound}. We can see from the subplots belonging to  DP-NAG-opt and DP-MASG-opt that, the re-adjustment prefers small $T$, and this selection indeed improves the performance. This further justifies the use of the optimized algorithms DP-NAG-opt and DP-MASG-opt, where the distribution of privacy preserving variance as well as the number of iterations are chosen automatically.

We also compare between the NAG-based schemes and their multi-stage versions. When the stepsize is chosen properly (\textit{c.f.} $c = 1$), both DP-NAG-opt and DP-MASG-opt perform very closely and outperform the others. However, DP-NAG-opt seems more robust to a poor selection of the stepsize (exampled by $c = 0.1$).
 
Finally, we compare the selections $m = 1000$ and $m = 100000$, where the first one corresponds to subsampling (with a rate of $1\%$), and the other corresponds to no subsampling. Firstly, we can see that, even when we subsample, optimizing $b_{t}$'s and $T$ according to Proposition \ref{prop: optimal sigma} does improve the performance of DP-NAG and DP-MASG significantly. (Recall that Proposition \ref{prop: optimal sigma} holds under no subsampling, yet its use under subsampling was discussed in Remark \ref{rem: optimize under subsampling}). Secondly, a comparison of Figures \ref{fig: comp_all_m_1000} and \ref{fig: comp_all_m_100000} on the whole shows that using full data improves the performance of the accelerated algorithms, especially for $c = 1$ (compare the lower halves of the figures). However, the difference does not seem to be by an order of magnitude. Since the additional randomness introduced by subsampling helps to decrease the required noise level for DP and using a sample instead of full data at each iteration is faster (in terms of per iteration running time), many DP methods in the literature consider stochastic algorithms in the DP context, where stochastic algorithms can improve the running time compared to deterministic algorithms. However, if the running time is not of concern for reaching a given privacy level, our experiments show that using full data results in a smaller bound on the objective error.

\section{Conclusions} \label{sec:conclusions}
In this paper, we presented two classes of differentially private optimization algorithms based on momentum averaging based on the heavy ball method and Nesterov's accelerated gradient method. We provided performance bounds for our algorithms for a given iteration budget while preserving a desired privacy level depending on the choice of the parameters (stepsize and the momentum). We showed that, for NAG, homogenous distribution of the privacy budget over all iterations, as typically done in the literature so far, is not the best way, and we propose a method to improve it.  Numerical experiments showed that the presented algorithms have advantages over their well-known straightforward versions.

Our analysis and methodology can be adapted to other forms of privacy to certain extents. For this, existence of a tractable formula for the noise parameter to satisfy a certain level of privacy is the key requirement. For example, a weaker form of $(\epsilon, \delta)$ differential privacy can be satisfied if the normal distribution is used for the privacy preserving noise, and, the required noise variance is well known \cite{dwork2014algorithmic}. Furthermore, provided no subsampling, privacy loss can be optimally distributed to the iterations of DP-NAG using a closed from formula as in Proposition \ref{prop: optimal sigma}, by exploiting the relation between zero-concentrated differential privacy of \cite{Bun_and_Steinke_2016} and $(\epsilon, \delta)$ differential privacy.

Our theoretical work formally investigates, both for DP-HB, DP-NAG, and the multi-stage version of DP-NAG, the effect of the algorithm parameters on the error bound. For DP-NAG and its multi-stage version, we also provide explicit formulas about how the variance of the gradient noise should be tuned at each stage to preserve a certain given level of privacy requirement, provided %with a bound on the initial error and 
the choice for the total number of iterations. % \comment{Burda initial error u bilmeden bu arguman  calisiyor muydu?} %For HB, 
%Methodologically, 
%For the multi-stager version of NAG, we also studied how the variance of the gradient noise can be tuned over stages to achieve the best performance within our performance bounds. 
However, in our setup, tuning of these parameters require knowledge about the constants $\mu$ and $L$. The Lipschitz constant $L$ can often be estimated from data using line search techniques (see e.g. \cite[Alg. 2]{schmidt2015non} or \cite{beck2009fast}). The strong convexity constant $\mu$ may also be known in some cases, for instance if a regularization term $\lambda \frac{\|x\|^2}{2}$ with $\lambda>0$ is added to a convex empirical risk minimization problem of the form \eqref{eq-finite-sum}, the strong convexity constant $\mu$ can be taken as $\lambda$. However, in general, $\mu$ may not be known and it may need to be estimated from data. As part of future work, it would be interesting to investigate whether restarting techniques developed for accelerated deterministic algorithms such as \cite{fercoq} which do not require the knowledge of the strong convexity constant a priori can be adapted to the privacy setting.

\appendix
\section{Definitions and Known Results} \label{app:0}

\begin{defn}[Strongly convex and smooth functions] \label{defn: st_convex}
A continuously differentiable function $F: \mathbb{R}^{d} \mapsto \mathbb{R}$ is called strongly convex with modulus $\mu > 0$ and L-smooth with a Lipschitz constant $L > 0$, if it satisfies
\[
\frac{\mu}{2}\norm{x-y}^2 \leq F(x) - F(y) -\nabla F(y)^{\top} (x-y) \leq \frac{L}{2}\norm{x-y}^{2}, \quad \forall x, y \in \mathbb{R}^{d}.
\]
The inequalities on the left and right hand sides separately define strong convexity and $L$-smoothness, respectively. Moreover, $\mathcal{S}_{\mu,L}(\RR^d)$ denotes the set of continuously differentiable functions that are strongly convex with modulus $\mu$ and L-smooth with $L$.
\end{defn}

\begin{defn}[$\epsilon$- Differential Privacy] \label{defn:DP} \textup{(Definition 1, \cite{Dwork_2008})} 
A randomized algorithm $\mathcal{A}$ with set of input datasets $\mathcal{Y}$ and range for its output $\mathcal{X}$ is $\epsilon$- differential private if for all datasets $Y, Y' \in \mathcal{Y}$ differing on at most one element i.e.\ $h(Y, Y') \leq 1$, and all measurable $O \subseteq \mathcal{X}$, it holds that $\mathbb{P}[A_{Y} \in O] \leq e^{\epsilon} \mathbb{P}[A_{Y'} \in O]$.
% \begin{align} \label{eq: defn_DP}
% \mathbb{P}[A_{Y} \in O] \leq e^{\epsilon} \mathbb{P}[A_{Y'} \in O].
% \end{align}
\end{defn}

\begin{defn}[Sensitivity]
\label{defn:Sensitivity} \textup{(Definition 3.1, \cite{dwork2014algorithmic})} 
For a function on datasets $\varphi:\mathcal{Y} \mapsto \mathbb{R}^{k}$, $k \geq 1$, the $L_{1}$-sensitivity of $\varphi$ is defined as
\begin{align} \label{eq: Sensitivity}
S_{1}^{\varphi} = \max_{Y, Y' \in \mathcal{Y}:  h(Y, Y') = 1} || \varphi(Y) - \varphi(Y') ||_{1}.
\end{align}
\end{defn}

\begin{thm}[Laplace mechanism] \label{thm: Laplace mechanism} \textup{(Theorem 1, \cite{Dwork_2008})} 
Given function $ \varphi :\mathcal{Y} \mapsto \mathbb{R}^{k}$, the mechanism $\mathcal{A}_{\varphi}$, which adds independently generated noise with Laplace distribution $\textup{Lap}(S_{1}^{\varphi} /\epsilon)$ to each of the $k$ output terms, is $\epsilon$-differentially private.
\end{thm}
The methods that we shall present in the subsequent sections use the Laplace mechanism at every iteration. Thus, we further need to quantify the privacy loss due to using a randomized algorithm repeatedly. 
\begin{thm}[Composition]\cite[Corollary 3.15]{dwork2014algorithmic} \label{thm: composition}
Let each algorithm $\mathcal{A}_{i}$ is $\epsilon_{i}$ differentially private. Then $(\mathcal{A}_{1}, \ldots, \mathcal{A}_{T})$, whose output is the concatenation of the outputs of the individual algorithms, is $\sum_{i=1}^{T} \epsilon_{i}$ differentially private.
\end{thm}

\begin{thm} \label{thm: subsampling_last} \textup{(Theorem 9, \cite{balle2018privacy})} 
Let $\mathcal{M}:\bigcup_{i = 1}^{\infty} \mathcal{Y}^{n} \rightarrow \mathcal{X}$ be a $\epsilon$-differentially private algorithm and the elements of $\mathcal{Y}^{n}$ be in the form of $y_{1:n}$. Then, an algorithm $\mathcal{M}_{m, n} :\mathcal{Y}^{n} \rightarrow \mathcal{X}$ that first selects a random subsample of $m$ items from its input data $y_{1:n} \in \mathcal{Y}^{n}$, by sampling without replacement, and then runs $\mathcal{M}$ on the subsample is $\epsilon'$-differentially private, where $\epsilon' = \ln \left({1+\frac{m}{n}(e^{\epsilon}-1)}\right)$.
% \begin{align} \label{eq: subsampling_privacy_main}
% \epsilon' = \ln \left({1+\frac{m}{n}(e^{\epsilon}-1)}\right).
% \end{align}
\end{thm}

\section{Omitted Proofs} \label{sec: Omitted Proofs}

We reserve this section for the proofs of several results in the main text.
\subsection{Proof of Proposition \ref{prop: bound on the variance of subsampling}} \label{app:A}

\begin{proof}{(Proposition \ref{prop: bound on the variance of subsampling})}
Fix $x_{t} = x \in \RR^{d}$ for the rest of the proof. Since $\textup{Cov}(\eta_{t} | x) = 2 b_{t}(x)^{2} I_{d}$, where $b_{t}(x) = S_{1}(x)/m\epsilon_{0}$, Assumption \ref{asmp: Bounded L1 sensitivity} implies that
\begin{equation} \label{eq: bound on the variance of DP noise} 
\| \textup{Cov}(\eta_{t} | x) \| = \frac{S_{1}(x)^{2} d}{m^{2} \epsilon_{0}^{2}} \leq \frac{S_{1}^{2} d}{m^{2} \epsilon_{0}^{2}}.
\end{equation}
Next, let $R = \textup{Cov} \left( \nabla F_{B}(x) | x \right)$, where $\nabla F_{B}(x)$ is the subsampling-based estimator of $\nabla F(x)$ when the indices in the subsample $B$ are sampled without replacement. From the unbiasedness property of $\nabla F_{B}(x)$, we have $R = \textup{Cov}(v_{t} | x)$. The diagonal terms in $R = [r_{i, j}]$ can be written as
\begin{equation} \label{eq: covariance of single gradient - diagonal terms} 
r_{k, k} = \sigma_{y, k}^{2} \frac{1}{m} \frac{n-m}{n - 1}, \quad k = 1, \ldots, d,
\end{equation}
where $\sigma_{y, k}^{2}$ is the population variance given by
\begin{equation} \label{eq: population variance} \sigma_{y, k}^{2} =
\frac{1}{n} \sum_{i = 1}^{n} \left( \frac{\partial f(x; y_{i})}{\partial x_{k}} - \frac{1}{n} \sum_{j = 1}^{n} \frac{\partial f(x; y_{j})}{\partial x_{k}} \right)^{2}.
\end{equation}
Let
\[
S_{1, k}(x) = \sup_{y, y' \in \mathcal{Y}} \left\vert \frac{\partial f (x; y) }{\partial x_{k}} - \frac{\partial f(x; y') }{\partial x_{k}} \right\vert, \quad k = 1, \ldots, d.
\]
The population variance in \eqref{eq: population variance} can then be bounded as $\sigma_{y, k}^{2} \leq S_{1, k}(x)^{2}/4$. Therefore, we bound $\| R \|$ by its trace as
\begin{align} 
 \| R \| &\leq \frac{1}{4} \left[ \sum_{k = 1}^{d} S_{1, k}(x)^{2}  \right] \frac{1}{m} \frac{n-m}{n  - 1} \nonumber \\
& \leq \frac{1}{4} \left[ \sum_{k = 1}^{d} S_{1, k}(x)  \right]^{2} \frac{1}{m} \frac{n-m}{n  - 1} \leq \frac{1}{4} S_{1}^{2} \frac{1}{m} \frac{n-m}{n  - 1}, \label{eq: bound covariance of subsampling}
\end{align}
where the last line is by Assumption \ref{asmp: Bounded L1 sensitivity}. Combining \eqref{eq: bound on the variance of DP noise} and \eqref{eq: bound covariance of subsampling} and using the triangle inequality for the matrix norm, we have the claimed bound on $|| \textup{Cov}(w_{t} | x) ||$.
\end{proof}

\subsection{Proof of Proposition \ref{prop: bd_HB}}
\label{app:Proof_2ndWay}
Recall, from Section \ref{sec: Error analysis of DP-SHB}, the dynamic system representation in \eqref{eq: DP-SHB dynamic system representation} and \eqref{eq: DP-SHB system matrices}, and define $\bar{F}: \mathbb{R}^{2d} \mapsto \mathbb{R}$ such that for $\xi_{t} = \begin{bmatrix} x_{t}^{\top} & x_{t-1}^{\top} \end{bmatrix}^{\top}$ we have $\bar{F}(\xi_{t}) = F(x_{t})$. Also, with  $X_{1}$, $X_{2}$ defined in Proposition \ref{prop: bd_HB}, we define $\tilde{X}_{1} = X_{1} \otimes I_{d}$ and $\tilde{X}_{2} = X_{2} \otimes I_{d}$. 

The following lemma is central to the proof of Proposition \ref{prop: bd_HB}.

\begin{lem} \label{lem: app_lemma}
Let $F \in \mathcal{S}_{\mu,L}(\mathbb{R}^{d})$ and consider the DP-SHB algorithm. Let $w_{t} = \eta_{t} + v_{t}$, the overall noise added to $\nabla F(x_{t})$ due to the Laplace mechanism and subsampling. Then for any $\rho \in (0,1)$, we have
\begin{align*}
& \mathbb{E}\left[ \begin{bmatrix} \xi_{t}-\xi^{\ast} \\ \nabla F(z_{t}) \end{bmatrix}^{\top} (\tilde{X}_{1} + (1 - \rho^{2}) \tilde{X}_{2}) \begin{bmatrix} \xi_{t}-\xi^{\ast} \\ \nabla F(z_{t}) \end{bmatrix} \right] \leq \rho^{2} \mathbb{E}[\bar{F}(\xi_{t}) - F^{\ast}] - \mathbb{E}[\bar{F}(\xi_{t+1}) - F^{\ast}] + \frac{L \alpha^{2}}{2} \mathbb{E}[\norm{w_{t}}^{2}]
\end{align*}
\end{lem}

\begin{proof}
One update rule of DP-SBH can be rewritten as:
\begin{equation} \label{eq: update of DP-SBH - rewritten}
x_{t+1} = (1 + \beta) x_{t}- \beta x_{t-1} - \alpha( \nabla F(x_{t}) + w_{t}),
\end{equation}
Using \eqref{eq: update of DP-SBH - rewritten}, we have
\begin{align}
x_{t}-x_{t+1} &= x_{t}-(1+\beta)x_{t}+\beta x_{t-1} + \alpha (\nabla F(x_t)+w_t) \nonumber \\
&= \beta(-x_{t} + x_{t-1}) + \alpha (\nabla F(x_t)+w_t)  \label{eq: xt minus xt+1}
\end{align}
Since $F \in \mathcal{S}_{\mu,L}(\mathbb{R}^{d})$, from Definition \ref{defn: st_convex}, using the inequality on the $L$-smoothness of $F$, we can write
\begin{equation} \label{eq: Lipschitz continuity for f} 
F(x_{t}) - F(x_{t+1}) \geq \nabla F(x_{t})^\top (x_{t}-x_{t+1}) - \frac{L}{2}\norm{x_{t+1}-x_{t}}^2.
\end{equation}
Combining \eqref{eq: Lipschitz continuity for f} with \eqref{eq: xt minus xt+1}, we obtain
\begin{align*}
F(x_{t}) - F(x_{t+1})  \geq \, & \nabla F(x_{t})^{\top}(-\beta (x_{t} - x_{t-1}) + \alpha (\nabla F(x_t)+w_t)) \\
& -\frac{L}{2}\norm{\beta(x_{t} - x_{t-1}) - \alpha (\nabla F(x_{t}) + w_{t})}^2 \\ 
=& -\beta (x_t-x_{t-1})^{\top} \nabla F(x_t) + \alpha \norm{\nabla F(x_t)}^2 + \alpha \nabla F(x_t)^{\top} w_t \\
& - \frac{L \beta^{2}}{2} \norm{x_t-x_{t-1}}^2 + L\alpha \beta (x_t-x_{t-1})^{\top} (\nabla F(x_t)+w_t)  - \frac{\alpha^2 L}{2}\norm{\nabla F(x_t)+w_t}^2 \\
=& \frac{1}{2} \begin{bmatrix} x_{t} -x_{t-1} \\ \nabla F(x_t) \end{bmatrix}^{\top} \tilde{D} \begin{bmatrix} x_t-x_{t-1} \\ \nabla F(x_t) \end{bmatrix} - \frac{L\alpha^{2}}{2} \norm{w_t}^2 + L \alpha [ \beta (x_t - x_{t-1}) \\ 
& - \alpha \nabla F(x_t)]^{\top} w_t  + \alpha \nabla F(x_t)^{\top} w_t,
\end{align*}
 where $\tilde{D} = D \otimes I_{d}$ is a $2 d \times 2 d$ matrix defined through $D = \begin{bmatrix} - L \beta^{2} & L\alpha \beta - \beta \\ L\alpha \beta - \beta & -\alpha^{2} L + 2 \alpha \end{bmatrix}$. Next, note that 
\[
\begin{bmatrix} x_{t} - x_{t-1} \\ \nabla F(h_t) \end{bmatrix} = \begin{bmatrix} I_{d} & -I_{d} & 0_{d} \\ 0_{d} & 0_{d} & I_{d} \end{bmatrix} \begin{bmatrix} x_{t} - x^{\ast} \\ x_{t-1} - x^{\ast} \\ \nabla F(x_t) \end{bmatrix}
\]
and
\[
\frac{1}{2} \begin{bmatrix} I_{d} & -I_{d} & 0_{d} \\ 0_{d} & 0_{d} & I_{d} \end{bmatrix}^{\top} \tilde{D} \begin{bmatrix} I_{d} & -I_{d} & 0_{d} \\ 0_{d} & 0_{d} & I_{d} \end{bmatrix} = \tilde{X}_{1}.
\]
Thus,
\begin{equation} \label{eq: X1}
\begin{split}
& F(x_{t}) - F(x_{t+1}) \geq\begin{bmatrix} x_t-x^{\ast} \\ x_{t-1}-x^{\ast} \\ \nabla F(x_{t}) \end{bmatrix}^{\top} \tilde{X}_{1} \begin{bmatrix} x_t-x^{\ast} \\  x_{t-1} - x^{\ast} \\ \nabla F(x_t) \end{bmatrix} \\
& \quad\quad\quad\quad -\frac{L \alpha^{2}}{2} \norm{w_{t}}^{2} + L \alpha \left[ \beta (x_t - x_{t-1}) - \alpha \nabla F(x_{t}) \right]^{\top} w_{t} - \alpha \nabla F(x_{t})^{\top} w_{t}
\end{split}
\end{equation} 
Similarly, by the inequality that gives strong convexity in Definition \ref{defn: st_convex}, we have
\begin{align*}
F(x^{\ast}) - F(x_t) &\geq \nabla F(x_t)^{\top} (x^{\ast} - x_t) + \frac{\mu}{2} \norm{x^{\ast} - x_{t}}^2 \\
&=  \frac{1}{2} \begin{bmatrix} x_t - x^{\ast} \\ x_{t-1} - x^{\ast} \\ \nabla F(x_t) \end{bmatrix}^{\top}
\begin{bmatrix} \mu I_{d} & 0_{d} & -I_{d} \\ 0_{d} & 0_{d} & 0_{d} \\ -I_{d} & 0_{d} & 0_{d} \end{bmatrix} 
\begin{bmatrix} x_{t} - x^{\ast} \\ x_{t-1} - x^{\ast} \\ \nabla F(x_t) \end{bmatrix} 
\end{align*}
The matrix in the middle is equal to $\tilde{X}_{2}$, so we can write
\begin{equation} \label{eq: X2} F(x^{\ast}) - F(x_{t}) \geq \frac{1}{2} \begin{bmatrix} x_{t} - x^{\ast} \\ x_{t-1} - x^{\ast} \\ \nabla F(x_{t}) \end{bmatrix}^{\top} \tilde{X}_{2}
\begin{bmatrix} x_{t} - x^{\ast} \\ x_{t-1} - x^{\ast} \\ \nabla F(x_{t}) \end{bmatrix}.
\end{equation}
Multiplying \eqref{eq: X2} by $(1-\rho^2)$ and adding to \eqref{eq: X1}, we obtain,
\begin{align*}
& \begin{bmatrix} \xi_{t}-\xi^* \\ \nabla F(x_{t}) \end{bmatrix}^{\top}  [\tilde{X}_{1} + (1 - \rho^{2}) \tilde{X}_{2}]  \begin{bmatrix} \xi_{t}-\xi^* \\ \nabla F(x_{t}) \end{bmatrix} 
\leq  \rho^2 [F(x_{t}) - F^{\ast}] - (F(x_{t+1}) - F^{\ast}) \\
&\quad\quad\quad\quad\quad\quad + \frac{L \alpha^{2}}{2} \norm{w_{t}}^{2} - L \alpha \left[ \beta (x_t - x_{t-1}) - \alpha \nabla F(x_{t}) \right] w_{t} - \alpha \nabla F(x_{t}) w_{t}.
\end{align*}
Taking the expectation, and applying $\mathbb{E}(w_{t})= 0$, we have the desired result.
\end{proof}

\begin{proof}{(Proposition \ref{prop: bd_HB})}
Lemma \ref{lem: app_lemma} is the counterpart Lemma 4.5 of \cite{aybat2018robust}, which is given for NAG. Hence, Lemma \ref{lem: app_lemma} allows us to extend the NAG results in \cite[Proposition 4.6 and Corollary 4.7]{aybat2018robust} for the DP-SHB. Finally, under Assumption \ref{asmp: Bounded L1 sensitivity}, we get the desired bound in our proposition.
\end{proof}

\subsection{Proof of Theorem \ref{thm: opt_quadraticHB}}
\label{app:C}

\begin{proof} Following the proof technique of \cite[Thm 12]{can2019accelerated}, we can write 
\begin{equation}
\begin{bmatrix} x_{t}-x^{\ast} \\ x_{t-1}-x^{\ast} \end{bmatrix}
=
M(\alpha, \beta) \begin{bmatrix} x_{t-1}-x^{\ast} \\ x_{t-2}-x^{\ast} \end{bmatrix} + \begin{bmatrix} - \alpha w_{t} \\ 0_{d} \end{bmatrix},
\label{eq-iter-quad-shb}
\end{equation}
where we have
\[
M(\alpha,\beta) = \begin{bmatrix} (1+\beta)I_{d}-\alpha Q & -\beta I_{d} \\ I_{d} & 0_{d} \end{bmatrix}.
\]
There also exists a permutation matrix $P$ such that
\begin{equation*} 
P M(\alpha,\beta) P^{\top} = \bar{T} := \begin{bmatrix} T_{1} & \cdots & 0 & 0 \\ 0 & T_{2} & \cdots & 0 \\ \vdots & \cdots & \ddots & \vdots  \\  0 & 0 & \cdots & T_{d} \end{bmatrix},
\end{equation*}
where $T_{i} = \begin{bmatrix} 1 + \beta - \alpha\lambda_{i} & -\beta \\ 1 & 0 \end{bmatrix}$, $1\leq i \leq d$, are $2\times 2$ matrices with eigenvalues
\begin{equation*}
a_{\lambda_i, \pm} = \frac{1+\beta-\alpha\lambda_{i}\pm\sqrt{(1+\beta-\alpha\lambda_{i})^{2}-4\beta}}{2}.
\end{equation*}
Therefore, for $t \geq 1$ we obtain
%\comment{Asagidaki denklemlerin bazilarinda ussu $t$ ile iterasyon belirtmis. Neden gerekli tam anlamadim?}
\begin{align*}
\left\Vert M(\alpha,\beta)^{t}\right\Vert =\left\Vert P^{\top} \bar{T}^{t} P \right\Vert &\leq \|P^{\top} \| \|P\| \max_{1\leq i\leq d}\left\Vert T_{i}^{t}\right\Vert = \max_{1\leq i\leq d}\left\Vert T_{i}^{t}\right\Vert,
\end{align*}
where we used the fact that $\|P\| = 1$ for a permutation matrix $P$. $T_{i}^{t}$ is a $2\times 2$ matrix, it has either semi-simple eigenvalues or a defective eigenvalue with a multiplicity two. In either case, it is well known that we can write $\|T_{i}^{t}\| \leq C_{i}^{t} \rho_{\lambda_{i}}^{t}$
%\comment{$\rho^t_{\lambda_i}$?} 
where $\rho_{\lambda_i} = \max \{ |a_{\lambda_i,+} |, \|a_{\lambda_i,-}\| \}$ is the spectral radius of $T_{i}^{t}$ and $C_{i}^{t} = \mathcal{O}(t)$. Then it follows that $\|M(\alpha,\beta)^{t}\| \leq C_{t} \rho^{t}$, where we take $ C_{t} = \max_{i} \{C_{i}^{t}\}$ and $\rho = \max_i\{\rho_{\lambda_i}\}$.
After a straightforward computation, we observe that $\rho_{\lambda}$ is a quasi-convex function of $\lambda$, therefore the function $\rho_\lambda$ attains its maximum as a function of $\lambda$ on the interval $[\mu, L]$ for either $\lambda=\mu$ or $\lambda=L$. Therefore, $\rho$ can also be written as 
\[
\rho = \max \{ \rho_{\lambda_\mu}, \rho_{\lambda_L} \} =  \max \{ | a_{\mu,+}|, |a_{\mu,-}|, |a_{L,+}|, |a_{L,-} | \}.
\]
Let
%\comment{Buradan asagiya $\leq$ ve (tanimlamadigimiz) $\preceq$ isaretlerini karisik kullanmisiz. Component-wise diyerek her yerde $\leq$ kullanabiliriz sanirim.}
$\hat{E}_{t} = \mathbb{E}\left[(\xi_{t}-\xi_{\ast})(\xi_{t}-\xi_{\ast})^{\top} | \xi_{t-1}\right]$. From \eqref{eq-iter-quad-shb}, we obtain the recursion
\begin{eqnarray*} 
\hat{E}_{t+1} &=& M(\alpha,\beta)\left[(\xi_{t}-\xi_{\ast})(\xi_{t}-\xi_{\ast})^{\top} \right]  M^{\top}(\alpha,\beta) + \left(\begin{array}{cc} \alpha^{2}\mathbb{E}(w_t w_t^{\top} | x_t) & 0_{d} \\ 0_{d} & 0_{d} \end{array}\right) \\
&\preceq& M(\alpha,\beta)\left[(\xi_{t}-\xi_{\ast})(\xi_{t}-\xi_{\ast})^{\top} \right]  M^{\top}(\alpha,\beta) + \left(\begin{array}{cc} \alpha^{2}\Sigma & 0_{d} \\ 0_{d} & 0_{d} \end{array}\right)
\end{eqnarray*}
Taking expectations with respect to $\xi_t$, we find
\begin{eqnarray*} \bar{E}_{t+1} 
&\preceq& M(\alpha,\beta)\bar{E}_t  M^{\top}(\alpha,\beta) + \left(\begin{array}{cc} \alpha^{2}\sigma_{T}^{2} I & 0_{d} \\ 0_{d} & 0_{d} \end{array}\right),
\end{eqnarray*}
where we let $\bar{E}_t := \mathbb{E}\left[(\xi_{t}-\xi_{\ast})(\xi_{t}-\xi_{\ast})^{\top} \right]$. We can also write
\begin{align*}
  \text{Tr}\left(\bar{E}_{t} \right) &\leq m(\alpha,\beta) +\left({M(\alpha,\beta)}\right)^{t} \bar{E}_{0} \left({M(\alpha,\beta)}^{\top}\right)^{t} \\
&- \sum_{j=t}^{\infty} M(\alpha,\beta)^{j}\left(\begin{array}{cc} \alpha^{2}c_W I & 0_{d} \\ 0_{d} & 0_{d} \end{array}\right)\left({M(\alpha,\beta)}^{\top} \right)^{j} \\
&\leq m(\alpha,\beta) +\left\Vert {M(\alpha,\beta)}^{t}\right\Vert^{2} \bar{E}_{0} +\sum_{j=t}^{\infty}\left\Vert {M(\alpha,\beta)}^{j}\right\Vert^{2}\alpha^{2}\Vert\Sigma\Vert \\
&\leq m(\alpha,\beta) + C_{t}^{2} \rho^{2t} \bar{E}_{0} +\alpha^{2} \sigma_{T}^{2}  C_{t}^{2}\frac{\rho^{2t}}{1-\rho^{2}},
\end{align*}
where we used the estimate
$\Vert {M(\alpha, \beta)}^{t}\Vert\leq C_{t} \rho^{t}$. This completes the proof.
\end{proof}

\subsection{Proof of Proposition \ref{prop: optimal sigma}} \label{sec: Proof of optimal b}

\begin{proof}
Observe from \eqref{eq: epsilon under subsampling} that, for $m = n$, we have $\varepsilon(S_{1}, b_{t}, n, n) = S_{1}/(b_{t} n)$. Hence, the optimization problem in \eqref{eq: optimization problem for NAG}  reduces to minimizing $\sum_{t = 1}^{T} a_{T, t} b_{t}^{2}$ over $b_{1}, \ldots b_{T}$ subject to $\sum_{t = 1}^{T} \frac{S_{1}}{n b_{t}} = \epsilon$.
% \begin{equation*}
% \min_{b_{1}, \ldots b_{T}} \quad \sum_{t = 1}^{T} a_{T, t} b_{t}^{2}, \quad \text{ subject to } \sum_{t = 1}^{T} \frac{S_{1}}{n b_{t}} = \epsilon.
% \end{equation*}
The above optimization problem can be solved by equating the derivative of the corresponding Lagrangian function 
\[
\sum_{t = 1}^{T} a_{T, t} b_{t}^{2} + \lambda \left(\sum_{t = 1}^{T} S_{1}/(n b_{t}) - \epsilon \right)
\]
with respect to $b_{1}, \ldots, b_{T}$, and $\lambda$ to $0$, which yields the system of $T + 1$ equations $2 a_{T, t} b_{t} = \frac{\lambda S_{1}}{n b_{t}^{2}}$ for $t = 1, \ldots, T$ and $\sum_{t = 1}^{T} \frac{S_{1}}{n b_{t}} = \epsilon$, solved at
\[
b_{t} = \left( \frac{\lambda S_{1}/n}{2 a_{T, t}} \right)^{1/3}, \quad \text{with} \quad \lambda = \frac{(S_{1}/n)^{2} \left[ \sum_{t = 1}^{T} (2 a_{T, t})^{1/3}\right]^{3}}{\epsilon^{3}} .
\]
Substituting $\lambda$ into $b_{t}$ yields the claimed solution. Finally, the bordered Hessian at the solution is a diagonal matrix, with $T$ negative values and a single $0$ on its diagonal.

\end{proof}
\section*{Acknowledgment}
Nurdan Kuru acknowledges support from TÜBİTAK 2214A Scholarship.

{\small 
\bibliographystyle{siamplain}
\bibliography{references}

\begin{thebibliography}{10}

\bibitem{abadi2016}
{\sc M.~Abadi, A.~Chu, I.~Goodfellow, H.~B. McMahan, I.~Mironov, K.~Talwar, and
  L.~Zhang}, {\em Deep learning with differential privacy}, in Proceedings of
  the 2016 ACM SIGSAC Conference on Computer and Communications Security, ACM,
  2016, pp.~308--318.

\bibitem{aybat2019masg}
{\sc N.~S. Aybat, A.~Fallah, M.~Gurbuzbalaban, and A.~Ozdaglar}, {\em A
  universally optimal multistage accelerated stochastic gradient method}, in
  Advances in Neural Information Processing Systems, 2019, pp.~8525--8536.

\bibitem{aybat2018robust}
{\sc N.~S. Aybat, A.~Fallah, M.~Gürbüzbalaban, and A.~Ozdaglar}, {\em Robust
  accelerated gradient methods for smooth strongly convex functions}, SIAM
  Journal on Optimization, 30 (2020), pp.~717--751.

\bibitem{balle2018privacy}
{\sc B.~Balle, G.~Barthe, and M.~Gaboardi}, {\em Privacy amplification by
  subsampling: {T}ight analyses via couplings and divergences}, in Advances in
  Neural Information Processing Systems, 2018, pp.~6277--6287.

\bibitem{Bassily_etal_2014}
{\sc R.~Bassily, A.~Smith, and A.~Thakurta}, {\em Private empirical risk
  minimization: {E}fficient algorithms and tight error bounds}, in 2014 IEEE
  55th Annual Symposium on Foundations of Computer Science, IEEE, 2014,
  pp.~464--473.

\bibitem{beck2009fast}
{\sc A.~Beck and M.~Teboulle}, {\em A fast iterative shrinkage-thresholding
  {A}lgorithm for linear inverse problems}, SIAM journal on imaging sciences, 2
  (2009), pp.~183--202.

\bibitem{Bun_and_Steinke_2016}
{\sc M.~Bun and T.~Steinke}, {\em Concentrated differential privacy:
  Simplifications, extensions, and lower bounds}, in Proceedings, Part I, of
  the 14th International Conference on Theory of Cryptography - Volume 9985,
  New York, NY, USA, 2016, Springer-Verlag New York, Inc., pp.~635--658.

\bibitem{can2019accelerated}
{\sc B.~Can, M.~G{\"u}rb{\"u}zbalaban, and L.~Zhu}, {\em Accelerated linear
  convergence of stochastic momentum methods in wasserstein distances}, in
  Proceedings of the 36th International Conference on Machine Learning, 2019.

\bibitem{Chaudhuri_Monteleoni_2009}
{\sc K.~Chaudhuri and C.~Monteleoni}, {\em Privacy-preserving logistic
  regression}, in Advances in Neural Information Processing Systems, 2009,
  pp.~289--296.

\bibitem{chaudhuri2011differentially}
{\sc K.~Chaudhuri, C.~Monteleoni, and A.~D. Sarwate}, {\em Differentially
  private empirical risk minimization}, Journal of Machine Learning Research,
  12 (2011), pp.~1069--1109.

\bibitem{Dwork_2006}
{\sc C.~Dwork}, {\em Differential privacy}, Springer-Verlag, 33rd International
  Colloquium on Automata, Languages and Programming, part II (ICALP 2006)
  (2006).

\bibitem{Dwork_2008}
{\sc C.~Dwork}, {\em Differential privacy: {A} survey of results}, in
  International Conference on Theory and Applications of Models of Computation,
  Springer, 2008, pp.~1--19.

\bibitem{dwork2014algorithmic}
{\sc C.~Dwork and A.~Roth}, {\em The algorithmic foundations of differential
  privacy}, Foundations and Trends{\textregistered} in Theoretical Computer
  Science, 9 (2014), pp.~211--407.

\bibitem{Dwork_and_Rothblum_2016}
{\sc C.~Dwork and G.~N. Rothblum}, {\em Concentrated differential privacy},
  tech. report, arXiv:1603.01887v2, mar 2016.

\bibitem{Dwork_etal_2010}
{\sc C.~Dwork, G.~N. Rothblum, and S.~Vadhan}, {\em Boosting and differential
  privacy}, in Foundations of Computer Science (FOCS), 2010 51st Annual IEEE
  Symposium on, IEEE, 2010, pp.~51--60.

\bibitem{fazlyab2018analysis}
{\sc M.~Fazlyab, A.~Ribeiro, M.~Morari, and V.~M. Preciado}, {\em Analysis of
  optimization algorithms via integral quadratic constraints: {N}onstrongly
  convex problems}, SIAM Journal on Optimization, 28 (2018), pp.~2654--2689.

\bibitem{fercoq}
{\sc O.~Fercoq and Z.~Qu}, {\em {Adaptive restart of accelerated gradient
  methods under local quadratic growth condition}}, IMA Journal of Numerical
  Analysis, 39 (2019), pp.~2069--2095.

\bibitem{flammarion2015averaging}
{\sc N.~Flammarion and F.~Bach}, {\em From averaging to acceleration, there is
  only a step-size}, in Conference on Learning Theory, 2015, pp.~658--695.

\bibitem{Foulds_et_al_2016}
{\sc J.~Foulds, J.~Geumlek, M.~Welling, and K.~Chaudhuri}, {\em On the theory
  and practice of privacy-preserving bayesian data analysis}, in Proceedings of
  the Thirty-Second Conference on Uncertainty in Artificial Intelligence,
  UAI'16, Arlington, Virginia, USA, 2016, AUAI Press, pp.~192--201.

\bibitem{gadat2018stochastic}
{\sc S.~Gadat, F.~Panloup, and S.~Saadane}, {\em Stochastic heavy ball},
  Electronic Journal of Statistics, 12 (2018), pp.~461--529.

\bibitem{hu2017dissipativity}
{\sc B.~Hu and L.~Lessard}, {\em Dissipativity theory for {N}esterov's
  accelerated method}, in Proceedings of the 34th International Conference on
  Machine Learning-Volume 70, JMLR. org, 2017, pp.~1549--1557.

\bibitem{hyland2019intrinsic}
{\sc S.~L. Hyland and S.~Tople}, {\em On the intrinsic privacy of stochastic
  gradient descent}, arXiv preprint arXiv:1912.02919,  (2019).

\bibitem{kifer2012private}
{\sc D.~Kifer, A.~Smith, and A.~Thakurta}, {\em Private convex empirical risk
  minimization and high-dimensional regression}, in Conference on Learning
  Theory, 2012, pp.~25--1.

\bibitem{lan2020first}
{\sc G.~Lan}, {\em First-order and Stochastic Optimization Methods for Machine
  Learning}, Springer, 2020.

\bibitem{lessard2016analysis}
{\sc L.~Lessard, B.~Recht, and A.~Packard}, {\em Analysis and design of
  optimization algorithms via integral quadratic constraints}, SIAM Journal on
  Optimization, 26 (2016), pp.~57--95.

\bibitem{loizou2017momentum}
{\sc N.~Loizou and P.~Richt{\'a}rik}, {\em Momentum and stochastic momentum for
  stochastic gradient, {N}ewton, proximal point and subspace descent methods},
  arXiv preprint arXiv:1712.09677,  (2017).

\bibitem{Mironov_2017}
{\sc I.~Mironov}, {\em R{\'e}nyi differential privacy}, 2017 IEEE 30th Computer
  Security Foundations Symposium (CSF),  (2017), pp.~263--275.

\bibitem{Mohammadi_et_al_2020}
{\sc H.~{Mohammadi}, M.~{Razaviyayn}, and M.~R. {Jovanovic}}, {\em Robustness
  of accelerated first-order algorithms for strongly convex optimization
  problems}, IEEE Transactions on Automatic Control,  (2020), pp.~1--1.

\bibitem{Jovanovic}
{\sc H.~{Mohammadi}, M.~{Razaviyayn}, and M.~R. {Jovanovic}}, {\em Robustness
  of accelerated first-order algorithms for strongly convex optimization
  problems}, IEEE Transactions on Automatic Control,  (2020), pp.~1--1.

\bibitem{nesterov1983}
{\sc Y.~Nesterov}, {\em A method of solving a convex programming problem with
  convergence rate ${O}(1/k^2)$}, Soviet Mathematics Doklady, 27 (1993),
  pp.~372--376.

\bibitem{Park_etal_2016}
{\sc M.~Park, J.~Foulds, K.~Chaudhuri, and M.~Welling}, {\em Variational bayes
  in private settings ({VIPS})}, Journal of Artificial Intelligence Research,
  68 (2020), pp.~109--157.

\bibitem{pichapati2019adaclip}
{\sc V.~Pichapati, A.~T. Suresh, F.~X. Yu, S.~J. Reddi, and S.~Kumar}, {\em
  Ada{C}lip: {A}daptive clipping for private {SGD}}, arXiv preprint
  arXiv:1908.07643,  (2019).

\bibitem{polyak1964some}
{\sc B.~T. Polyak}, {\em Some methods of speeding up the convergence of
  iteration methods}, USSR Computational Mathematics and Mathematical Physics,
  4 (1964), pp.~1--17.

\bibitem{polyak1987introduction}
{\sc B.~T. Polyak}, {\em Introduction to optimization. optimization software},
  Inc., Publications Division, New York, 1 (1987).

\bibitem{ramezani2018stability}
{\sc A.~Ramezani-Kebrya, A.~Khisti, and B.~Liang}, {\em On the stability and
  convergence of stochastic gradient descent with momentum}, arXiv preprint
  arXiv:1809.04564,  (2018).

\bibitem{rubinstein2009learning}
{\sc B.~I.~P. Rubinstein, P.~L. Bartlett, L.~Huang, and N.~Taft}, {\em Learning
  in a large function space: {P}rivacy-preserving mechanisms for {SVM}
  learning}, arXiv preprint arXiv:0911.5708,  (2009).

\bibitem{schmidt2015non}
{\sc M.~Schmidt, R.~Babanezhad, M.~Ahmed, A.~Defazio, A.~Clifton, and
  A.~Sarkar}, {\em Non-uniform stochastic average gradient method for training
  conditional random fields}, in artificial intelligence and statistics, 2015,
  pp.~819--828.

\bibitem{shokri2015privacy}
{\sc R.~Shokri and V.~Shmatikov}, {\em Privacy-preserving deep learning}, in
  Proceedings of the 22nd ACM SIGSAC conference on computer and communications
  security, ACM, 2015, pp.~1310--1321.

\bibitem{song2020characterizing}
{\sc S.~Song, O.~Thakkar, and A.~Thakurta}, {\em Characterizing private clipped
  gradient descent on convex generalized linear problems}, arXiv preprint
  arXiv:2006.06783,  (2020).

\bibitem{vapnik2013nature}
{\sc V.~Vapnik}, {\em The nature of statistical learning theory}, Springer
  science \& business media, 2013.

\bibitem{yan2018unified}
{\sc Y.~Yan, T.~Yang, Z.~Li, Q.~Lin, and Y.~Yang}, {\em A unified analysis of
  stochastic momentum methods for deep learning}, arXiv preprint
  arXiv:1808.10396,  (2018).

\bibitem{Yang_etal_2016}
{\sc T.~Yang, Q.~Lin, and Z.~Li}, {\em Unified convergence {A}nalysis of
  stochastic momentum methods for convex and non-convex optimization}, arXiv
  preprint arXiv:1604.03257,  (2016).

\bibitem{yu2019differentially}
{\sc L.~Yu, L.~Liu, C.~Pu, M.~E. Gursoy, and S.~Truex}, {\em Differentially
  private model publishing for deep learning}, in 2019 IEEE Symposium on
  Security and Privacy (SP), IEEE, 2019, pp.~332--349.

\bibitem{zhang2012functional}
{\sc J.~Zhang, Z.~Zhang, X.~Xiao, Y.~Yang, and M.~Winslett}, {\em Functional
  mechanism: {R}egression analysis under differential privacy}, Proceedings of
  the VLDB Endowment, 5 (2012), pp.~1364--1375.

\bibitem{Zhang_etal_2017}
{\sc J.~Zhang, K.~Zheng, W.~Mou, and L.~Wang}, {\em Efficient private {E}{R}{M}
  for smooth objectives}, arXiv preprint arXiv:1703.09947,  (2017).

\end{thebibliography}
}
\end{document}